%% file: main.tex
\documentclass[twoside,11pt]{article}

\usepackage[english]{babel}
\usepackage[utf8]{inputenc}
\usepackage[T1]{fontenc}
\usepackage[a4paper,left=2.5cm,right=2.5cm,top=2cm,bottom=3cm]{geometry} 
\usepackage{tikz}
\usetikzlibrary[topaths]

\usepackage{subfig}

\usepackage{enumitem} 
\usepackage{tabto} 
\usepackage{lipsum}

\usepackage{wrapfig}

\usepackage{changepage}

\usepackage{soul}  

\usepackage{cite}

\usepackage{microtype}
\usepackage{graphicx}

\usepackage{algorithm}
\usepackage{algcompatible} 

\usepackage{amsmath, mathrsfs, amssymb, amsthm, mathtools} 

\usepackage{color}

\usepackage{changepage}

\newtheorem{theorem}{Theorem}
 
\newtheorem{definition}{Definition}
\newtheorem{lemma}{Lemma}
\newtheorem{remark}{Remark}

\newtheorem{proposition}{Proposition}
\newtheorem{assumption}{Assumption}

\newtheorem{observation}{Observation} 

\newtheorem*{theorem*}{Theorem}
\newtheorem*{example*}{Example} 
\newtheorem*{definition*}{Definition}
\newtheorem*{lemma*}{Lemma}
\newtheorem*{remark*}{Remark}
\newtheorem*{corollary*}{Corollary}
\newtheorem*{proposition*}{Proposition}
\newtheorem*{assumption*}{Assumption}
\newtheorem*{claim*}{Claim}

\def\submission{0}

\input{macro.tex}

\begin{document}

\title{From Random Search to Bandit Learning in Metric Measure Spaces} 

\author{Chuying Han\footnote{19300180141@fudan.edu.cn} \quad Yasong Feng\footnote{ysfeng20@fudan.edu.cn} \quad Tianyu Wang\footnote{wangtianyu@fudan.edu.cn}}

\maketitle

\begin{abstract} 
    Random Search is one of the most widely-used method for Hyperparameter Optimization, and is critical to the success of deep learning models. Despite its astonishing performance, little non-heuristic theory has been developed to describe the underlying working mechanism. This paper gives a theoretical accounting of Random Search. We introduce the concept of \emph{scattering dimension} that describes the landscape of the underlying function, and quantifies the performance of random search. We show that, when the environment is noise-free, the output of random search converges to the optimal value in probability at rate $ \widetilde{\mathcal{O}} \left( \left( \frac{1}{T} \right)^{ \frac{1}{d_s} } \right) $, where $ d_s \ge 0 $ is the scattering dimension of the underlying function. When the observed function values are corrupted by bounded $iid$ noise, the output of random search converges to the optimal value in probability at rate $ \widetilde{\mathcal{O}} \left( \left( \frac{1}{T} \right)^{ \frac{1}{d_s + 1} } \right) $. In addition, based on the principles of random search, we introduce an algorithm, called BLiN-MOS, for Lipschitz bandits in doubling metric spaces that are also endowed with a probability measure, and show that under mild conditions, BLiN-MOS achieves a regret rate of order $ \widetilde{\mathcal{O}} \left( T^{ \frac{d_z}{d_z + 1} } \right) $, where $d_z$ is the zooming dimension of the problem instance. 
\end{abstract}

\input{tex/intro-new}

\input{tex/rs}

\input{tex/noisy-rs} 
\input{tex/zooming}

\input{tex/mos-noisy}

\input{tex/conslusion}

\input{tex/bib}
% \section{}

% \bibliographystyle{plain} 
% \bibliography{biblio} 

\appendix 

\input{tex/appendix} 

\end{document}

%% file: macro.tex
\DeclareMathAlphabet{\mathpzc}{OT1}{pzc}{m}{it}

\makeatletter
\newcommand{\oset}[3][0ex]{%
  \mathrel{\mathop{#3}\limits^{
    \vbox to#1{\kern1\ex@
    \hbox{$\scriptstyle#2$}\vss}}}}
\makeatother

\newcommand{\wh}{\widehat}

\newcommand{\wt}[1]{\widetilde{#1}}

\renewcommand{\Pr}{ \mathbb{P} }

\newcommand{\A}{ \mathcal{A} }

\renewcommand{\[}{ \left[ }
\renewcommand{\]}{ \right] }

\renewcommand{\(}{ \left( }
\renewcommand{\)}{ \right) }

\renewcommand{\O}{ \mathcal{O} }

\newcommand{\X}{ \mathcal{X} }

\renewcommand{\O}{ \mathcal{O} }

\renewcommand{\[}{ \left[ }
\renewcommand{\]}{ \right] } 

\newcommand{\R}{ \mathbb{R} } 

\newcommand{\D}{ \mathcal{D} }

%% file: tex/intro-new.tex
\section{Introduction}

Random Search \cite{bergstra2011algorithms,bergstra2012random} is one of the most widely-used method for HyperParameter Optimization (HPO) in training neural networks. 
Despite its astonishing performance, little non-heuristic theory has been developed to describe the underlying working mechanism. HPO problems can be formulated as a zeroth-order optimization problem. More specifically, for HPO problem, we seek to solve 
\begin{align*} 
    \min_{x \in S } f (x) , 
\end{align*} 
where $S$ is the feasible set, $f$ is the objective function, and only zeroth-order information of $f$ is available. For HPO tasks, $f$ is typically nonsmooth nonconvex.

% In this paper, we formulate HPO tasks as zeroth-order optimization problems. 
To quantify the performance of random search, we introduce a new concept called \emph{scattering dimension} that describes the landscape of the objective function $f$. In general, random search performs well on functions with small scattering dimension. More specifically, we show that, in noise-free environments, the optimality gap of random search with $T$ random trials converges to zero in probability at rate $\wt{\mathcal{O}} \( \( \frac{1}{T} \)^{\frac{1}{d_s}} \) $, where $ d_s$ is the scattering dimension of the underlying function. When the observed function values are corrupted by bounded $iid$ noise, the output of random search converges to the optimal value in probability at rate $ \wt{\mathcal{O}} \( \( \frac{1}{T} \)^{ \frac{1}{d_s + 1} } \) $. 

% We also 
% random search benefits from small scattering dimension. 

% In addition, we introduce a new algorithm, called Batched Lipschitz Narrowing with Maximum Ordered Statistics (BLiN-MOS). BLiN-MOS extends random search by adaptively narrowing the search space based on random search results. We show that, with total time horizon $T$, optimality gap of BLiN-MOS is asymptotically bounded by $ O \( \( \frac{T}{\log T} \)^{-\frac{1}{d_s}} \) $ with high probability, where $ d_s $ is the scattering dimension of the objective function $f$. 

The scattering dimension relates to the classic concept of zooming dimension \cite{kleinberg2008multi,bubeck2008tree}. We illustrate that, for polynomials with the stationary point in the domain, the sum of the scattering dimension and the zooming dimension equals the ambient dimension of the space. Since sum of polymonials approximate smooth functions, and maximum/minimum of polynomaials can create nonsmoothness, this result shall hold true for a larger class of functions. The scattering dimension requires an additional probability measure to be endowed on the space --- the zooming dimension is well-defined over compact doubling metric spaces, whereas the scattering dimension requires there to be an additional probability measure over this space. 
% In other words, the scattering dimension depends , since a compact doubling metric space may not be endowed with a canonical probability measure (Definition \ref{def:canon-prob}). 

% Whether we can define the scattering dimension without such a measure might be an interesting open problem. 

Also, we introduce a Lipschitz bandit algorithm called Batched Lipschitz Narrowing with Maximum Order Statistics (BLiN-MOS). When the reward samples are corrupted by bounded and positively supported noise, BLiN-MOS achieves a regret of order $ \wt{\mathcal{O}} \( T^{\frac{d_z}{d_z+1}} \) $ in metric spaces with a probability measure, where $d_z$ is the zooming dimension \cite{kleinberg2008multi,bubeck2008tree} of the problem instance. Also, $\mathcal{O} (\log \log T)$ rounds of communications are sufficient for achieving this regret rounds. 
% This regret rate breaks the information-theoretical lower bound for Lipschitz bandits in metric spaces \citep{kleinberg2008multi}. 
% Our results suggest an intrinsic gap between metric spaces and metric measure spaces from an algorithmic perspective, since the upper bound $ \wt{\mathcal{O}} \( T^{\frac{d_z}{d_z+1}} \) $ breaks the information-theoretical lower bounds for Lipschitz bandits \citep{kleinberg2008multi}. In addition, we show that BLiN-MOS only needs $ \mathcal{O} (\log \log T) $ rounds of communications to achieve a regret of order $ \wt{\mathcal{O}} \( T^{\frac{d_z}{d_z+1}} \) $ in metric measure spaces. 

In summary, our main contributions are 
\begin{enumerate}[leftmargin=*] 
    \item We provide the first non-heuristic analysis of the random search algorithm, which is widely used in HPO. 
    \item We introduce the concept of scattering dimension that describes the landscape of the underlying function. In addition, scattering dimension quantifies the convergence rates of random search. 
    \item We introduce a new Lipschitz bandit algorithm called BLiN-MOS. Under mild conditions, BLiN-MOS achieves a regret rate of order $ \wt{\mathcal{O}} \( T^{ \frac{d_z}{d_z+1} } \) $ in metric spaces endowed with a probability measure, where $d_z$ is the zooming dimension. In addition, only $ \mathcal{O} \(\log \log T\) $ rounds of communications are needed to achieve this rate. 
    % Our results suggest that there is a fundamental difference between metric spaces and metric measure spaces, since this regret rate in metric measure spaces breaks the known lower bound in metric spaces. 
    % show that 
\end{enumerate}

\section{Related Works} 

The spirit of random search could probably date back to classical times when searching became necessary for scientific discoveries or engineering designs. In modern machine learning, the principles and procedures of random search were first formally studied by \cite{bergstra2011algorithms,bergstra2012random}, motivated by the surging need of hyperparameter tuning in neural network training. 
Ever since the proposal of random search for HPO in machine learning, it has been a standard benchmark algorithm for subsequent works focusing on HPO; See (e.g., \cite{feurer2019hyperparameter}) for an exposition. Despite ubiquitousness of random search, little non-heuristic theory has been developed to analyze its performance. Perhaps the most related reasoning of the astonishing performance of random search comes from the Bayesian optimization community \cite{wang2016bayesian}. In their paper \cite{wang2016bayesian}, Wang et al. wrote 
\begin{quote} 
    \emph{``[T]he rationale [behind random search's performance] [is] that points sampled uniformly at random in each dimension can densely cover each low-dimensional subspace. As such, random search can exploit low effective dimensionality without knowing which dimensions are important.''} 
\end{quote} 

Other than this comment, no reasoning for the working mechanism of random search is known. In this paper, we fill this gap by introducing the concept of \emph{scattering dimension} that describes the landscape of the objective function. Using this language, we can precisely quantify the performance of random search. In addition, we design BLiN-MOS that extends random search using recent advancements in Lipschitz bandits. 
% \citep{fenglipschitz}. 

Lipschitz bandit problems have been prosperous since its proposal under the name of ``Continuum-armed bandit'' \cite{agrawal1995continuum}. Throughout the years, various researchers have contributed to this field (e.g., \cite{kleinberg2005nearly,kleinberg2008multi,bubeck2008tree,bubeck2011x,magureanu2014lipschitz,christina2019nonparametric,krishnamurthy2019contextual,lu2019optimal,fenglipschitz}). To name a few relatively recent results, \cite{magureanu2014lipschitz} derived a concentration inequality for discrete Lipschitz bandits; The idea of robust mean estimators \cite{bickel1965some, alon1999space, 10.1214/11-AIHP454,bubeck2013bandits} was applied to the Lipschitz bandit problem to cope with heavy-tail rewards \cite{lu2019optimal}. \cite{pmlr-v134-podimata21a} considered the adversarial Lipschitz bandit problem, and introduced an exponential-weights \cite{LITTLESTONE1994212,auer2002nonstochastic,arora2012multiplicative} algorithm that adaptively learns the overall function landspace. 
% Recently, \cite{fenglipschitz} studied Lipschitz bandit problems in a batched feedback environment, and designed an algorithm that is both regret-optimal and batch-optimal. 

Perhaps the most important works in modern Lipschitz bandit literature are \cite{kleinberg2008multi,bubeck2008tree}. In \cite{kleinberg2008multi,bubeck2008tree}, the concept of zooming dimension was introduced, and algorithms that near-optimally solve the Lipschitz bandit problem were introduced. In particular, the optimal algorithm achieves a regret rate of order  $ \wt{\mathcal{O}} \( T^{\frac{d_z + 1}{d_z + 2}} \) $ in (compact doubling) metric spaces, where $d_z$ is the zooming dimension. In addition, matching lower bounds in metric spaces are proved. In this paper, Lipschitz bandit problems in metric measure spaces are considered. In particular, we propose an algorithm, called BLiN-MOS, that achieves a regret rate of order $ \wt{\mathcal{O}} \( T^{\frac{d_z}{d_z+1}} \) $ in metric spaces with a probability measure. Our results show that, under certain conditions, the previous $ \wt{\mathcal{O}} (T^{\frac{d_z+1}{d_z+2}}) $ regret rate can be improved. In addition, we 
show that $\mathcal{O} (\log \log T)$ rounds of communications are sufficient for BLiN-MOS to achieve regret of order $ \wt{\mathcal{O}} ( T^{\frac{d_z}{d_z+1}} ) $ in metric measure spaces. 
% algorithm to achieve the 

\noindent \textbf{Paper Organization.} The rest of this paper is organized as follows. In Section \ref{sec:rs}, we introduce the concept of scattering dimension and use it to characterize the performance of the random search algorithm. Section \ref{sec:zooming} is dedicated to basic properties of scattering dimension. In Section \ref{sec:mos}, we introduce an algorithm for stochastic continumm-armed bandit.

%% file: tex/rs.tex
\section{Understanding Random Search via the Scattering Dimension} 

\label{sec:rs}

% As discussed above, the zooming dimension alone is not sufficient to efficiently instruct the difficulty level of finding a near-optimal arm. To address this issue, we introduce the concept of \emph{scattering dimension} in Definition \ref{def:ord}. 

Scattering dimension of a function $f$, as the name implies, describes how likely a randomly scattered point in the domain hits a near-optimal point. 
Consider a compact metric measure space $(\X, \D, \nu)$ where $\nu$ is a probability measure defined over the Borel $\sigma$-algebra of $\D$. 
% Let $ \mu $ be the Borel measure induced by the metric $\D$. That is, the $\mu$ is defined over the $\sigma$-algebra generated by metric ball open balls defined by $\D$. Since $\X$ is compact, $\mu$ be to normalized to a probability measure $\Pr$. Now, consider a function $ f : \X \to \R $. 
For any set $S \subseteq \X$, define $ f_S^{\max} = \sup_{x \in S} f (x)$ and $f_S^{\min} = \inf_{x \in S} f (x)$. For any closed subset $ Z \subseteq \X  $, let $\mathcal{F}_Z$ be the Borel $\sigma$-algebra on $ Z $ (with respect to $\D$). Let $\Pr$ be the probability law defined with respect to $\nu$, and let $ X_Z $ be the random variable such that $ \Pr \( X_Z \in E \) = \Pr \( E \) $ for any $E \in \mathcal{F}_Z $. 

% let $\Pr|_{{Z}} $ be the probability measure over the Borel $\sigma$-algebra of $Z$ such that $ \Pr|_{{Z}} (E) = \Pr  $. In addition, for any closed ball $ B \subseteq \X $, define the $X_B$ to be the random variable such that $ \Pr|_{B} ( E ) = \frac{  }{ \Pr (  ) } $ 

With the above notations, we formally define scattering dimension below. 

% Consider the 

\begin{definition} 
    \label{def:scattering} 
    % We assume that the following items are true. 
    % \begin{itemize}
    Let $ ( \X, \D, \Pr ) $ be a compact metric measure space, where $\Pr$ is a probability measure. 
    % Let $X_B$ be the random variable 
    % , and let $\Pr$ be the probability measure induced by the metric $\D$. 
    For a Lipschitz function $f : \X \to \R$ whose maximum is obtained at $x^* \in \X$, define the scattering dimension $ d_s $ of $f$ as  
    \begin{align*} 
        d_s := \inf \{ \tilde{d} \ge 0 :&\; \exists \kappa \in (0,1], \text{ such that } \Pr \( f ( X_B ) < f_B^{\max} -  \alpha ( f_B^{\max} - f_B^{\min} ) \) 
        \le 
        1 - \kappa \alpha^{ \tilde{d} }, \\ 
        &\; \forall \text{ closed ball } B \subseteq \X \text{ with $x^* \in B$},\;  \forall \alpha \in (0, {1} ] \} . 
    \end{align*} 

    In addition, we define the scattering constant $\kappa_s$ to be 
    \begin{align*}
        \kappa_s := \max \{ \kappa \in (0,1] :&\;  \Pr \( f ( X_B ) < f_B^{\max} -  \alpha ( f_B^{\max} - f_B^{\min} ) \)  
        \le 
        1 - \kappa \alpha^{d_s}, \\ 
        &\; \forall \text{ closed ball } B \subseteq \X \text{ with $x^* \in B$},\;  \forall \alpha \in (0, {1} ] \} . 
    \end{align*}
\end{definition} 

% \textcolor{red}{$\alpha \in (0,1]$ or $\alpha \in (0,\frac{1}{2}]$?}

\begin{remark}
    Throughout the rest of the paper, we consider only Lipschitz functions for which the scattering dimension is defined with a positive scattering constant $\kappa_s$. 
\end{remark}

\subsection{Important Special Cases}  

We consider a special case $ ([0,1]^d, \| \cdot \|_\infty, \nu) $, where $\nu$ is the Lebesgue measure over $ \[ 0,1 \]^d $. For algorithmic purpose, the space $ ([0,1]^d, \| \cdot \|_\infty, \nu) $ is of great importance, since random search is usually implemented using uniformly random samples over a cube. In this case, the scattering dimension is defined as follows. 

\begin{definition} 
    \label{def:scattering-simple} 
    % We assume that the following items are true. 
    % \begin{itemize} 
    Consider the space $ ([0,1]^d, \| \cdot \|, \nu) $ where $ \nu $ is the Lebesgue measure over $ \[ 0,1 \]^d $. Let $\Pr$ be the probability measure defined with respect to $\nu$. 
    % Let $ ( \X, \D, \Pr) $ be a compact metric measure space. 
    % , and let $\Pr$ be the probability measure induced by the metric $\D$. 
    % Let $ X_q $ be a random variable 
    For a Lipschitz function $f : [0,1]^d \to \R$ whose maximum is obtained at $x^* \in [0,1]^d $, define the scattering dimension $ d_s $ of $f$ as  
    \begin{align*} 
        d_s := \inf \{ \tilde{d} \ge 0 :&\; \exists \kappa \in (0,1], \text{ such that } \Pr \( f ( X_q ) < f_q^{\max} -  \alpha ( f_q^{\max} - f_q^{\min} ) \) 
        \le 
        1 - \kappa \alpha^{ \tilde{d} }, \\ 
        &\; \forall \text{ closed cube } q \subseteq [0,1]^d \text{ with $x^* \in q$},\;  \forall \alpha \in (0,1] \} . 
    \end{align*} 

    In addition, we define the scattering constant $\kappa_s$ to be 
    \begin{align*}
        \kappa_s := \max \{ \kappa \in (0,1] :&\;  \Pr \( f ( X_q ) < f_q^{\max} -  \alpha ( f_q^{\max} - f_q^{\min} ) \)  
        \le 
        1 - \kappa \alpha^{d_s}, \\ 
        &\; \forall \text{ closed cube } q \subseteq \X \text{ with $x^* \in q$},\;  \forall \alpha \in (0,1] \} . 
    \end{align*} 
\end{definition}

% We will illustrate the trade-off between zooming and scattering dimension via the following example function. 
% illustrate the characteristics of zooming dimension via the following example function $g_p$. 
% As its definition implies, functions with larger zooming dimension 

\subsubsection{Scattering Dimension of Norm Polynomials} 
\label{sec:scatter-gp}

To illustrate how the scattering dimension describes the function landscape, we consider the following function $g_p$. 
For $p\ge 1$, let $g_p(x) = 1 - \frac{1}{p} \| x \|_\infty^p$ be a function defined over $[0,1]^d$ for some $d \ge 1$. Figure \ref{fig:demo} illustrates $g_p $ with $p=1,3,5,10$ and $d = 1$. We have the following proposition that illustrates how scattering dimension describes the function landscape. The proof of Proposition \ref{prop:gp} can be found in the Appendix.

\begin{proposition} 
    \label{prop:gp} 
    Let $p \ge 1$, and let $ g_p (x) : [0,1]^d \to \R$ be defined as $ g_p (x) = 1 - \frac{1}{p} \| x \|_\infty^p $. The scattering dimension of $g_p$ is $d_s = \frac{d}{p}$ and the scattering constant of $g_p$ is $\kappa_s = 1$. 
\end{proposition}

In the space $ ([0,1]^d, \| \cdot \|_\infty, \nu ) $, the scattering dimension of many important functions can be explicitly calculated. 

% More such results can be found in \textcolor{red}{(where.)}

% \begin{align*} 
    
% \end{align*}  

\subsection{The Random Search Algorithm} 

The random search algorithm is a classic and concise algorithm. Its procedure is summarized in Algorithm \ref{alg:rs}. 

\begin{algorithm}[H]
    \caption{Random Search} 
    \label{alg:rs} 
    \begin{algorithmic}[1]  
        \STATE \textbf{Input.} The space $ (  \X, \D , \nu ) $. Zeroth-order oracle to the unknown function $f : \X \to \R$. Total number of trials $T$. 
        % Arm set $\mathcal{X}=[0,1]^d$; total budget $T$; Number of selected arms $N\leq T$.
        \STATE Randomly select $T$ points $ \{X_i\}_{i=1}^T \subseteq \X$, 
        % $\mathcal{X}_s=\{x_i\}_{i=1}^N\subseteq\mathcal{X}$, 
        where each $X_i$ is governed by the law of $ \nu $.
        % \STATE Choose an arm $\widetilde{x}^*\in\mathcal{X}_s$ 
        % according to some policy. 
        \STATE \textbf{Output} $ Y_T^{\max} = \max \{ f (X_1), f (X_2), \cdots, f (X_T) \} $. 
    \end{algorithmic} 
\end{algorithm} 

The performance of Random Search can be quantified by the following theorem, whose proof can be found in the Appendix. 
\begin{theorem} 
    \label{thm:rs}
    Consider a compact metric measure space $(\X, \D, \mu)$. Let $\Pr$ be the probability measure defined with respect to $\mu$. Let $f : \X \to \R$ be a Lipschitz function with Lipschitz constant $L$. Let $f^*$ be the maximum value of $f$ in $\X$. Then for any $ T \in \mathbb{N} $, the output of Algorithm \ref{alg:rs} satisfies 
    \begin{align*}
        \Pr \( f^* - Y_T^{\max} \le L \Theta \( \frac{1 - e^{ \frac{\log \epsilon }{T} }}{ \kappa_s } \)^{ \frac{1}{d_s} } \) \ge 1 - \epsilon, \forall \epsilon \in (0,1) , 
    \end{align*}
    where $d_s$ is the scattering dimension of $f$ and $ \Theta = \max_{x,x' \in \X} \D (x,x')  $ is the diameter of $ \X $. 
    % \textcolor{red}{(roughly)} 
\end{theorem}

Theorem \ref{thm:rs} implies that, with probability exceeding $ 1 - \epsilon $, the final optimality gap of random search is upper bounded by $ L \Theta \( \frac{1 - e^{ \frac{\log \epsilon }{T} }}{ \kappa_s } \)^{ \frac{1}{d_s} } $, which is of order $ \mathcal{O} \( L \Theta \( \frac{ \log (1/\epsilon) }{ \kappa_s T } \)^{ \frac{1}{d_s} } \) $. 

In addition, we have the following asymptotic result for random search, whose proof can be found in the Appendix. 

\begin{theorem} 
    \label{thm:rs-asymp}
    Consider a compact metric measure space $(\X, \D, \mu)$ with diameter $\Theta$. 
    Let $\Pr$ be the probability measure defined with respect to $\mu$. 
    Let $f : \X \to \R$ be a Lipschitz function with Lipschitz constant $L$. Let $f^*$ be the maximum value of $f$ in $\X$. 
    Let $ \omega_T$ be an arbitrary sequence such that $ \lim\limits_{T \to \infty} \omega_T = \infty $ and $\omega_T = o (T)$. Then $ \( \frac{ T }{\omega_T} \)^{\frac{1}{d_s}} \( f^* - Y_{T}^{\max} \) \overset{P}{\to} 0 $.  
\end{theorem}

Note that $\omega_T$ can be slowly varying. For example, $ \omega_T $ can go to infinity at the same rate as $ \log \log \log \log T $.

%% file: tex/noisy-rs.tex
\subsection{The Random Search Algorithm in Noisy Environments}

Till now, the environments we consider are noiseless. Next we consider an environment where the function values are corrupted by $i.i.d.$ noise. Let $W$ denote the real-valued noise random variable. In addition, we assume the law of $W$ is absolutely continuous with respect to the Lebesgue measure. 
Let $f_W$ and $F_W$ be the $pdf$ and $cdf$ of $W$. The random search algorithm can be summarized as follows. 
% \begin{align*}
    \begin{algorithm}[H]
    \caption{Random Search (in Noisy Environments)} 
    \label{alg:rs-noisy} 
    \begin{algorithmic}[1]  
        \STATE \textbf{Input.} The space $ (  \X, \D , \nu ) $. Noisy zeroth-order oracle to the unknown function $f : \X \to \R$. Total number of trials $T$. 
        % Arm set $\mathcal{X}=[0,1]^d$; total budget $T$; Number of selected arms $N\leq T$.
        \STATE Randomly select $T$ points $ \{X_i\}_{i=1}^T \subseteq \X$, 
        % $\mathcal{X}_s=\{x_i\}_{i=1}^N\subseteq\mathcal{X}$, 
        where each $X_i$ is governed by the law of $ \nu $.
        % \STATE Choose an arm $\widetilde{x}^*\in\mathcal{X}_s$ 
        % according to some policy. 
        \STATE Observe noisy function values $ \{ f (X_1) + W_1, f (X_2) + W_2, \cdots, f (X_T) + W_T \} $.
        % and output $  $
        \STATE \textbf{Output} $ Y_T^{\max} = \max \{ f (X_1) + W_1, f (X_2) + W_2, \cdots, f (X_T) + W_T \} $, and $ X_T^* \in \arg\max_{X_i} \{ f (X_i) + W_i \} $. /* Here we only observe the noise-corrupted functions $\{ f (X_i) + W_i \}_{i=1}^T $. */
    \end{algorithmic} 
\end{algorithm} 

The performance guarantee for Algorithm \ref{alg:rs-noisy} is found in Theorem \ref{thm:rs-noisy}. The proof of Theorem \ref{thm:rs-noisy} can be found in the Appendix.

\begin{theorem} 
    \label{thm:rs-noisy} 
    Let $ W_i $ be $i.i.d.$ random variables that are positively supported on interval $[a,b]$ for some $a,b \in \R$. Then for any sequence $\omega_T$ that satisfies $ \omega_T \to \infty $, the output of Algorithm \ref{alg:rs-noisy} satisfies $ \( \frac{T}{\omega_T} \)^{\frac{1}{d_s + 1}} \( f^* + b - Y_T^{\max} \) \overset{P}{\to} 0 $, where $f^*$ is the maximum of $f$ over $\X$.  
\end{theorem} 

% \end{align*}

% ---------------------------------------

% ---------------------------------------

%% file: tex/zooming.tex
\section{Zooming Dimension versus Scattering Dimension}
\label{sec:zooming}

The term zooming dimension was coined to characterize the landscape of the underlying function \cite{kleinberg2008multi,bubeck2011x}. Scattering dimension is closely related to zooming dimension. Before proceeding, we review the definition of zooming dimension. 

% The term zooming dimension was coined to characterize the landscape of the underlying function \cite{kleinberg2008multi,bubeck2011x}. In fact, scattering dimension is closely related to zooming dimension. Before proceeding to their connections, we review the definition of zooming dimension. 

Consider the compact doubling metric space $( \X, \D )$, where $ \D $ is a doubling metric. Let $d$ be the doubling dimension of the metric space $ (\X, \D) $. 
Consider a Lipschitz function $f : \X \in \R$. The set of $r$-optimal arms is defined as $ S (r) = \{ x \in \X : \max_{z \in \X} f (z) - f (x) \le r \} $. 
The $r$-zooming number of $f$ is defined as 
\begin{align*} 
    N_r :=& \;  \min \Big\{ | F | :  
    F = \{ B : \text{$B$ is a metric ball of radius $r$} \} 
     \text{ and } \cup_{B \in F } = S (2r)  \Big\} . 
    % \# \{B: \text{$B$ is a metric ball of radius $r$ and $B \subset S( 4 r )$}\} . 
\end{align*} 
In words, $N_r$ is the $r$-covering number of the set $S(2r)$. 

The zooming dimension is defined as
\begin{align*} 
    {d}_z := \min \{   d ' \geq0: \exists a>0, 
    \text{such that } N_r \le a r^{-d'}, \; \forall r > 0 \}. 
\end{align*} 

% The zooming dimension was 

% Below we state a requirement of zooming dimension that has not formally appeared in existing literature. 
To foster the discussion, we state below a property of zooming dimension. 
\begin{proposition} 
    \label{prop} 
    There exists a function $f: \X \to \R$ such that the zooming dimension is infinite if there exists $z^* \in \X$ and $r > 0$ such that the metric ball $ B (z^*, 2r) $ cannot be covered by finitely many balls of radius $r$. 
    % , or (2) the function $f$ is flat. 
\end{proposition} 

\begin{proof} 
    Suppose there exists $z^* \in \X$ and $r > 0$ such that for some $r$, the metric ball $ B (z^*, 2r) $ cannot be covered by finitely many balls of radius $r$. 
    Define $f : \X \to \R$ so that $f (x) = - \D (x, z^*)$, and the zooming dimension of $f$ with respect to $\D$ is infinity. 
    % Then $ S (2r) = B (z^* , 2r) $ and it can pack infinitely many balls of radius $r$, leading to a contradiction. 
\end{proof} 

% As shown above, the zooming dimension is not always finite. 
% In fact, this is the main reason that the authors of \cite{pmlr-v119-wang20q,fenglipschitz} restricted their attention to the space $ ([0,1]^d, \| \cdot \|_\infty) $. In this space, the metric balls are cubes, enabling easier implementation of algorithms. In addition, in this space, the doubling dimension equals the ``ambient dimension''. 

\subsection{The Curse and Blessing of Zooming Dimension} 

% The term zooming dimension was coined to characterize the landscape of the underlying function \citep{kleinberg2008multi,bubeck2011x}. 
According to classical theory, functions with smaller zooming dimension are easier to optimize, because larger zooming dimension means more arms are similar and harder to distinguish. 
On the other hand, larger zooming dimension also implies a larger near-optimal region, which increases the chance of finding a near-optimal region by a random sample in the domain. 
To sum up, a function with large zooming dimension has two contrasting features: 
% \begin{itemize} 

% \noindent $\bullet$ \textbf{(The Curse of Zooming Dimension)} 
% When the zooming dimension $d_z$ is large, the difference between arms is small. From this perspective, the problem of finding the exact maximum is hard when $d_z$ is large. 

% \noindent $\bullet$ \textbf{(The Blessing of Zooming Dimension)} When the zooming dimension $d_z$ is large, the region of near-optimal arms is large. From this perspective, the problem of finding an approximate maximum is easy when $d_z$ is large. 
% \end{itemize} 

\begin{itemize}[leftmargin=*]
\item 
% \noindent $\bullet$ 
\textbf{(The Curse of Zooming Dimension)} 
When the zooming dimension $d_z$ is large, the difference between arms is small. From this perspective, the problem of finding the exact maximum is hard when $d_z$ is large. 
\item 
% \noindent $\bullet$ 
\textbf{(The Blessing of Zooming Dimension)} When the zooming dimension $d_z$ is large, the region of near-optimal arms is large. From this perspective, the problem of finding an approximate maximum is easy when $d_z$ is large. 
\end{itemize} 

Existing theory mainly focuses on the cursing side of zooming dimension. In particular, in the language of bandit, existing works show that the regret rate deteriorates as $d_z$ increases \cite{kleinberg2008multi,bubeck2008tree,fenglipschitz}. The blessing side of zooming dimension has been largely overlooked. The blessing effect is largely captured by scattering dimension, as previously discussed. 

\subsection{A Numeric Example} 

% We will illustrate the trade-off between zooming and scattering dimension via the following example function. 
% illustrate the characteristics of zooming dimension via the following example function $g_p$. 
% As its definition implies, functions with larger zooming dimension 
% Let $g_p(x) = 1 - \frac{1}{p} \| x \|_\infty^p$ be a function defined over $[0,1]^d$ for some $d \ge 1$ and $p\ge1$. Figure \ref{fig:demo} illustrates $g_p $ with $p=1,3,5,10$ and $d = 1$. 
Recall the example functions we considered in Section \ref{sec:scatter-gp}. 
We observe the following property for such function $g_p$. 
\begin{observation} 
    \label{obs:gp} 
    Consider the metric space $([0,1]^d, \| \cdot \|_\infty)$. Let $p \ge 1$, and let $ g_p (x) : [0,1]^d \to \R$ be defined as $ g_p (x) = 1 - \frac{1}{p} \| x \|_\infty^p $. The zooming dimension of $g_p$ is $d_z = \frac{p-1}{p} d $. 
\end{observation} 

\if\submission0
\begin{figure}
    \centering
    \includegraphics[width = 0.45 \textwidth]{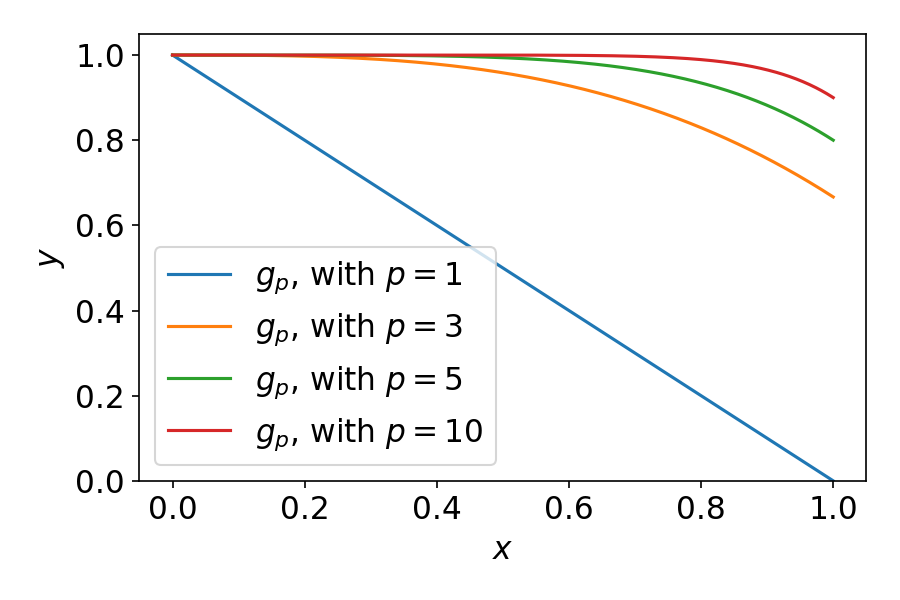} 
    \caption{The plots of $ g_p $ with $p = 1,3,5,10$ over $[0,1]$.} 
    \label{fig:demo} 
\end{figure}
\fi 

\if\submission1
% \begin{wrapfigure}{l}{0.5\textwidth} 
% \includegraphics[width=0.9\linewidth]{figures/gp.png} 
% \caption{The plots of $ g_p $ with $p = 1,3,5,10$ over $[0,1]$.} 
% \label{fig:demo} 
% \end{wrapfigure} 
\begin{figure}
    \centering
    \includegraphics[width = 0.45 \textwidth]{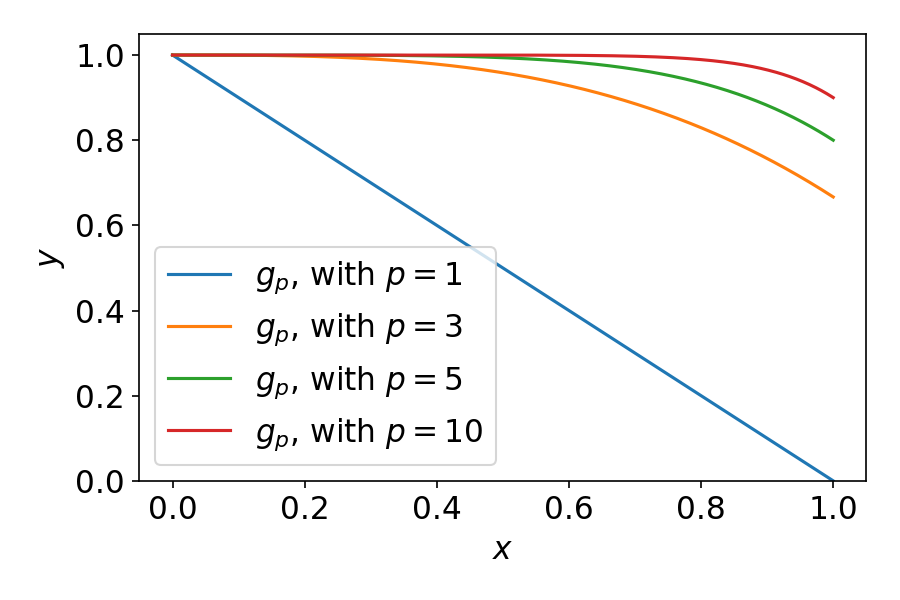} 
    \caption{The plots of $ g_p $ with $p = 1,3,5,10$ over $[0,1]$.} 
    \label{fig:demo} 
\end{figure}
\fi

% Since $ g_p (x) = 1 - \frac{1}{p} \( \max_i x_i \)^p $, we know 

Clearly $ g_p (x) $ is 1-Lipschitz. By definition of $g_p$, the $2 \rho$-optimal arms form a cube $ [ 0, ( 2 p \rho )^{\frac{1}{p}} ]^d $. We need $ \( 2 p \)^{\frac{d}{p}} \rho^{\frac{1-p}{p} d } $ cubes of edge-length $ \rho $ to cover $ [0, (2 p \rho)^{\frac{1}{p}} ]^d $. Thus the zooming dimension of $ g_p $ is $d_z = \frac{p-1}{p} d$.

As $p$ approaches infinity, the zooming dimension converges to $d$ and the zooming constant converges to 1. According to classic theory \cite{kleinberg2008multi,bubeck2011x}, the performance of an optimal-seeking bandit algorithm deteriorates as $p$ increases. The intuition behind this theory is that when $d_z$ is large, the arms are nearly indistinguishable. However, the function landscape of $g_p$ flattens as $p \to \infty $, and the problem of finding a near-optimal point becomes easier. For function $g_p$, the numeric sum of $d_s$ and $d_z$ equals the ambient dimension $ d $. In the Appendix we provide a more formal statement of the relation between $d_z$ and $d_s$. 
% when $p$ is large, $g_p$ has a small scattering dimension, and a random point in $[0,1]^d$ is close to optimal. 

\subsection{Scattering Dimension Requires a Probability Measure} 

As we have shown in Proposition \ref{prop}, the zooming dimension is defined for a doubling metric. On contrary, the scattering dimension needs there to be a probability measure over the space. 
% Indeed, we 
% That is, the scattering dimension makes an axiomatically stronger requirement on the underlying space. 
The need for specifying probability measure on the underlying space is simple: 
% \begin{itemize}[leftmargin=*]
% 
% \noindent $\bullet$ 
As per how scattering dimension is defined (Definition \ref{def:scattering}), there needs to be a well-defined probability measure over the space so that the sampling actions can be performed. 
% \end{itemize} 

At this point, some natural questions may arise: \emph{Is there a ``default'' choice for such probability measure?} If there exists a ``default'' probability measure on a doubling metric space, then we can use this ``default'' probability to define the scattering dimension. So, can we always find such a ``default'' probability measure? 
% However, 

It turns out that, a ``default'' probability measure may not exists in a general metric space, and we really have to specify a probability measure. By a ``default'' probability measure we mean the following. 
\begin{definition}[Canonical probability measure] 
    \label{def:canon-prob}
    % \textcolor{red}{define canonical measure.} 
    Let $ (X,d) $ be a compact metric space. For any $\epsilon > 0$, let $ N_\epsilon \subseteq X $ be an $\epsilon$-net of $ (X,d) $. That is, $N_\epsilon$ satisfies: 1. $ \cup_{x \in N_\epsilon} B (x, \epsilon) \supseteq X $ where $ B (x, \epsilon) $ is the (open) ball of radius $\epsilon$ centered at $x$; 2. Any set that satisfies item 1 has cardinality no smaller than $N_\epsilon$. With respect to $ N_\epsilon $, define a probability measure such that for any Borel set $ Y \subseteq X $, $ \mu_\epsilon := \frac{|Y \cap N_\epsilon|}{|N_\epsilon|} $. 
    If there exists a measure $\mu$ such that for all Borel set $Y \subseteq B$, $\mu (Y) = \lim_{\epsilon \to 0} \mu_\epsilon (Y) $ for all choice of $\epsilon$-nets, then $\mu$ is called the canonical probability measure of $ (X,d) $. 
\end{definition} 

\begin{remark} 
    The ``canonical probability measure'' (Definition \ref{def:canon-prob}) is different from the Haussdorff measure, since this measure (Definition \ref{def:canon-prob}) is not a metric outer measure. 
\end{remark} 

The reason that such a probability measure is called canonical with respect to the metric is: If a set $Y$ is ``large'' with respect to the metric, then the set $Y$ is ``large'' with respect to the measure $\mu$. 
However, such a ``canonical probability measure'' may not always exist. 
Below we provide an example that show a canonical probability measure does not always exist. 
This example is inspired by \cite{278375}. 

Consider the set 
\begin{align*} 
    S = \Bigg\{ \sum_{n=1}^\infty \frac{a_{2n-1}}{2^{2n-1}} + \frac{a_{2n}}{3^{2n}} : & \; a_{2n-1} \in \{ 0,1 \} 
    \text{ and } a_{2n} \in \{ 0,1,2 \} \Bigg\}.  
\end{align*}  
For any $ x = \sum_{n=1} \frac{a_{2n-1}}{2^{2n-1}} + \frac{a_{2n}}{2^{2n}} \in S$ and  $ x' = \sum_{n=1} \frac{a_{2n-1}'}{2^{2n-1}} + \frac{a_{2n}'}{2^{2n}} \in S $, define the confluence of $x$ and $x'$ as 
\begin{align*} 
    (x|x') := \inf 
    \{ k \in \mathbb{N} : a_i = a_i' \text{ for all } i \le k \}. 
\end{align*} 
Then define a distance metric: $ d (x,x') = 2^{-(x|x')} $ and by convention $ 2^{-\infty} = 0 $. Note that elements in $S$ can be represented by sequences $(a_1,a_2, \cdots)$, $a_{2n-1} \in \{ 0,1 \}$ and $a_{2n} \in \{ 0,1,2 \}$. 
% Also, we use a finite tuple $ (a_1, a_2, \cdots, a_N) $ as a shorthand for $ (a_1, a_2, \cdots, a_N, \underbrace{0,0,\cdots}_{\text{all zeros}}) $. 
Consider the set of digits ending with infinitely many consecutive 1's: 
\begin{align*} 
    Y = \cup_{k \in \mathbb{N}_+} \{ ( a_1, a_2, \cdots, a_k, 1,1,\cdots ) \in S \} , 
\end{align*} 
and $\epsilon$-nets for $\epsilon = 2^{-n}$: 
\begin{align*} 
    N_{2^{-2n+1}} := \{ \( a_1, a_2, \cdots, a_{2n-1}, 1,1,1,\cdots \) \in S \}, 
\end{align*} 
and 
\begin{align*} 
    N_{2^{-2n}} := \{ \( a_1, a_2, \cdots, a_{2n} , 0,0,0,\cdots\) \in S \}, 
\end{align*}   

Then we have 
\begin{align*} 
    \mu_{2^{-2n+1}} (S \setminus Y) = 0 , \quad \text{ and } \quad 
    \mu_{2^{-2n}} (Y) = 0 . 
    % \frac{ 0 }{2^{n} \cdot 3^{n}} . 
\end{align*} 
This shows that a canonical probability measure (Definition \ref{def:canon-prob}) does not always exist. Therefore, we really have to specify a probability measure over the space.

%% file: tex/mos-noisy.tex
\section{The BLiN-MOS Algorithm} 
\label{sec:mos}

The Batched Lipschitz Narrowing (BLiN) algorithm was recently introduced as an optimal solver for Lipschitz bandits with batched feedback \cite{fenglipschitz}. In particular, BLiN simultaneously achieves state-of-the-art regret rate, with optimal communication complexity. 
% Later, 
% the authors found that the 
% BLiN was applied to HPO tasks, and \cite{feng2023hpo} found that when the zooming dimension is small, BLiN outperforms existing HPO methods, and founds a better noise schedule for diffusion models, compared to existing HPO methods. 
In this section, we propose an improved version of BLiN: Batched Lipschitz Narrowing with Maximum Order Statistics (BLiN-MOS). BLiN-MOS uses the BLiN framework, and integrates in the advantages of random search. The motivation behind BLiN-MOS is as follows. For a reward maximizing task such as the bandit learning, we need not estimate the average payoff in each region. Instead, it suffices to estimate the best payoff in each region. The algorithm procedure of BLiN-MOS is summarized in Algorithm \ref{alg:large-d}. The notations and conventions for Algorithm \ref{alg:large-d} can be found in Section \ref{sec:lip-conv}.

\begin{algorithm}[ht] 
	\caption{Batched Lipschitz Narrowing with Max Order Statistics (BLiN-MOS)} 
	\label{alg:large-d} 
	\begin{algorithmic}[1]  
		\STATE \textbf{Input.} Arm set $\A=[0,1]^d$. Time horizon $T$. Probability parameter $\epsilon$. Number of batches $M$. Scattering parameter $ \beta $. /* We let $\beta = 1$ to avoid clutter. */
		\STATE \textbf{Initialization.} Edge-length sequence $\{r_m\}_{m=1}^{M+1}$; The first grid point $t_1=0$; Equally partition $\mathcal{A}$ to $r_1^d$ subcubes and define $\mathcal{A}_{1}$ as the collection of these subcubes. 
		% \STATE Compute $n_m=\frac{16\log T}{r_m^2}$ for $m=1,\cdots,B+1$. 
		\FOR{$m=1,2,\cdots,M$} 
                \STATE Compute $ n_m = \frac{\log ( \epsilon)}{\log \( 1 - \frac{ \kappa_s }{ \sqrt{2\pi} (d_s + 1) } \exp \( - \frac{1 }{2} \) r_m \) }$. 
                % $ n_m = \frac{\log (1/\epsilon) }{ \log \(1 - \frac{\kappa_s \kappa_p }{d_s + 1} r_m \beta^{d_s + 1} \) }$. %, where $k > 0 $. 
                % $\alpha_m = 1-(\frac{256 A_\epsilon }{ n_m })^{1/4} $, $\rho_m = 6 \kappa_+ \( \frac{A_\epsilon}{n_m} \)^{1/4} 2^{-m}$. 
                % \FOR{cube $q \in \mathcal{A}_{m} $} 
                \STATE Uniformly randomly sample $n_m$ points $ x_{q,1}, x_{q,2}, \cdots, x_{q,n_m} $ from $q \in \mathcal{A}_m$. Let $y_{q,1}, y_{q,2}, \cdots, y_{q,n_m}$ be the associated noisy samples. 
                % , and compute $ Y_{q, n_m } $. 
                \STATE Compute $ Y_{q,n_m} = \max \{ y_{q,1} , y_{q,2} , \cdots, y_{q,n_m } \} $ and $ Y_m^{\max} := \max_{q \in \mathcal{A}_m } Y_{q,n_m} $. 
                \STATE Let $\mathcal{A}_{m}' = \{ q \in  \mathcal{A}_m : {Y}_m^{\max} - {Y}_{q,n_m} \le r_m \}$. /* Elimination step */ 
                % \textcolor{red}{(some tricks for $r_m$??? )} 
                \STATE Dyadically partition the cubes in $\mathcal{A}_{m}'$, and collect these cubes to form $ \mathcal{A}_{m+ 1 } $. 
                \STATE Compute $t_{m+1}=t_m+2(r_{m}/r_{m+1})^d\cdot|\mathcal{A}_{m }' | \cdot n_{m+1}$. 
                \IF{$t_{m+1}\geq T$} 
                    \STATE Finish the remaining pulls arbitrarily. \textbf{Terminate} the algorithm. 
                \ENDIF 
		\ENDFOR 
        % \STATE \textbf{Output:} $ Y_{M}^{\max} $. 
	\end{algorithmic} 
\end{algorithm}

% ----------------------------------------------------

\subsection{Notations and Conventions} 
\label{sec:lip-conv} 

We conform to the following conventions that are common for Lipschitz bandit problems. Following \cite{kleinberg2008multi,slivkins2011contextual}, we assume that the function of interests is $1$-Lipschitz. Following \cite{pmlr-v134-podimata21a,fenglipschitz}, we restrict our attention to the doubling metric space $ \( [0,1]^d, \| \cdot \|_{\infty} \) $. This metric space is doubling with balls being cubes. 

% For algorithmic purposes, 
% We conform to the following conventions that are common for Lipschitz bandit problems. Following \cite{kleinberg2008multi,slivkins2011contextual}, we assume that the function of interests is $1$-Lipschitz. Following \cite{pmlr-v119-wang20q,fenglipschitz}, we restrict our attention to the doubling metric space $ \( [0,1]^d, \| \cdot \|_{\infty} \) $. This metric space is doubling with balls being cubes. 

\begin{assumption}
    \label{assumption:measure}
    For all algorithmic analysis in $ \( [0,1]^d, \| \cdot \|_{\infty} \) $, we endow this space with the Lebesgue measure and define the scattering dimension and scattering constant with respect to this metric measure space. 
\end{assumption}

In addition, we assume that $\beta = 1$ in Algorithm \ref{alg:large-d} solely for the sake of cleaner presentation. 
% Following \cite{fenglipschitz}, we use a dyadic version of zooming dimension to avoid clutter. 

Define the set of $r$-optimal arms as $ S (r) = \{ x \in \A : \Delta_x \le r \} $, where $\Delta_x = f^* - f(x)$. For any $r=2^{-i}$, the decision space $[0,1]^d$ can be equally divided into $2^{di}$ cubes with edge length $r$, which are referred to as dyadic cubes. 
% Henceforth, all cubes are standard cubes unless otherwise noted. 
The $r$-zooming number is defined as 
% \textcolor{red}{($6r$ or $3r$?)}
\begin{align*} 
    N_r := \#\{C: \text{$C$ is a standard cube with edge length $r$ and $C\subset S(6 r)$}\} . 
\end{align*} 
The zooming dimension is then defined as 
\begin{align*} 
    d_z := \min \{  d\geq0: \exists a>0,\;N_r \le ar^{-d},\;\forall r=2^{-i} \text{ for some $i \in \mathbb{N}$}\}. 
\end{align*}
Moreover, we define the zooming constant $\kappa_z$ as the minimal $a$ to make the above inequality true for $d_z$, $\kappa_z=\min\{  a>0:\;N_r \le ar^{-d_z},\;\forall r=2^{-i} \text{ for some $i \in \mathbb{N}$} \}$.

\subsection{Analysis of BLiN-MOS}

The performance guarantee for Algorithm \ref{alg:large-d} with $r_m = 2^{-m}$ can be found in Theorem \ref{thm:mos-noise-total}, after Assumption \ref{assumption:noise}. 

\begin{assumption} 
    \label{assumption:noise}
    We assume that all observations are corrupted by $iid$ copies of noise random variable $W$, and that $ W $ is strictly positively supported on $ [-1,1] $. Let $\kappa_p$ lower bound the density of $W$. 
\end{assumption} 

In Assumption \ref{assumption:noise}, the choice of $[-1,1]$ is purely for the purpose of cleaner presentation. Our results generalize to other compact intervals. 
 
\begin{theorem}
    \label{thm:mos-noise-total}
    Instate Assumptions \ref{assumption:measure} and \ref{assumption:noise}. 
    % Consider 
    Let $r_m = 2^{-m}$. Let $T$ be the total time horizon and let $\epsilon = \frac{1}{T^2}$ in BLiN-MOS. Let the number of batches $M \ge \frac{\log \left( \frac{\kappa_s \kappa_p T (2^{d_z} - 1)}{\kappa_z d_z (d_s+1) \cdot 2^{d_z + 1} \log T} \right) }{(d_z+1) \log 2} $. 
    % Let the reward samples be corrupted by  
    Then with probability exceeding $1-\frac{2}{T}$, the total regret of BLiN-MOS in noise environments satisfies
    \begin{equation*}
        R(T) \leq c \cdot T^{\frac{d_z }{d_z + 1 }} \cdot \left( \log T \right)^{\frac{1}{d_z + 1 }},
    \end{equation*}
    where $R(T)$ denote the total regret up to time $T$, and $c$ is a constant independent of $T$. 
\end{theorem} 

% \begin{remark} 
%     The noise assumption (Assumption \ref{assumption:noise}) in Theorem \ref{thm:mos-noise-total} does not affect the lower bound analysis in the classic theory \citep{kleinberg2008multi,bubeck2008tree} at all. Indeed, problem instances with $iid$ rewards positively supported on a closed interval (e.g., uniform distribution plus a bump at a specified point) can satisfy Theorem \ref{thm:mos-noise-total} in a metric measure space, and satisfy the lower bound in the classic theory in a metric space. Therefore, Theorem \ref{thm:mos-noise-total} shows that the lower bound in metric spaces does not hold in metric measure spaces considered in Theorem \ref{thm:mos-noise-total}. 
%     % Compared to the theory of classic Lipschitz bandit algorithms, the only additional assumptions in Theorem \ref{thm:mos-noise-total} are 1. Assumption \ref{assumption:noise} and 2. that the there is , so that the scattering dimension is properly defined. Note that 
% \end{remark} 

Next we present the proof of Theorem \ref{thm:mos-noise-total}, which relies on Lemmas \ref{lem:not-eli}, \ref{lem:event} and \ref{lem:opt-gap}. Proofs of Lemmas \ref{lem:not-eli}, \ref{lem:event} and \ref{lem:opt-gap} are in the Appendix. 

\begin{lemma} 
    \label{lem:not-eli} 
    Instate Assumptions \ref{assumption:measure} and \ref{assumption:noise}. Let $\epsilon = \frac{1}{T^2}$ in BLiN-MOS (Algorithm \ref{alg:large-d}). 
    With probability exceeding $ \( 1 - \frac{1}{T^2} \)^T $, the optimal arm $x^*$ is not eliminated throughout the algorithm execution of BLiN-MOS. 
\end{lemma}

\begin{lemma} 
    \label{lem:event}
    Instate Assumptions \ref{assumption:measure} and \ref{assumption:noise}. Consider a BLiN-MOS run with a fixed time horizon $T$ and with $\epsilon = \frac{1}{T^2}$. Let 
    \begin{align*} 
        \mathcal{E} = \left\{ \text{$x^*$ is not eliminated during a $T$-step run } \right\} \cap \left\{ f^* + 1 - Y_{m}^{\max} \le r_m, \; \forall m = 1,2,\cdots, M \right\} ,
    \end{align*} 
    where $ f^* = \max_{x \in [0,1]^d} f (x) $. 
    Then $\Pr \( \mathcal{E} \) \ge 1 - \frac{2}{T} $. 
    % \textcolor{red}{(It seems that the notation $f^*$ above has not yet been defined.)}
\end{lemma}

\begin{lemma}
    \label{lem:opt-gap} 
    Instate Assumptions \ref{assumption:measure} and \ref{assumption:noise}. Let $\Delta_x := f^* - f (x)$ denote the optimality gap of arm $x$. 
    Under event $\mathcal{E}$, for any $m = 1,2,\cdots,M$, any $q_m \in \mathcal{A}_m$ and any $x\in q_m$, $\Delta_x$ satisfies
    $$
    \Delta_x \leq 3 r_{m-1}. 
    $$
    % where $L$ is the Lipschitz constant of $f(x)$. 
\end{lemma}

% ----------------------------------------------------

% ----------------------------------------------------

\begin{proof}[Proof of Theorem \ref{thm:mos-noise-total}]  
    Let $\mathcal{E}$ in Lemma \ref{lem:event} be true. 
    By Lemma \ref{lem:opt-gap} and the definition of zooming dimension (Definition \ref{sec:zooming}), we know that $| \mathcal{A}_m | \le \kappa_z 2^{m d_z}$ and any cube $q_m \in \mathcal{A}_m$ is contained in $S(3 r_{m-1}) = S(6 r_m)$. 

    % \textcolor{red}{The notation of total batch is not uniform, use $M$ or $M^*$?}
    
    As we play in total $M$ batches, it holds that
    % \begin{align*} 
    $
        R(T) \leq \sum_{m=1}^{M} | \mathcal{A}_m | \cdot n_m \cdot 6 r_m
        + 6 r_m \cdot T. $
    % \end{align*} 
    Since $ | \mathcal{A}_m | \le \kappa_z 2^{m d_z}$ and $n_m \le \frac{(d_s+1) \log (1/\epsilon)}{\kappa_s \kappa_p } 2^{ m} $ (See Proposition \ref{prop:basic} in the Appendix), the above implies 
    % \begin{equation*} 
    \begin{align*} 
        R(T) & \leq \frac{12 \kappa_z (d_s+1) \log T}{\kappa_s \kappa_p} \sum_{m = 1}^{M}  2^{ d_z  m} + 6 r_{M} \cdot T \\
        & \leq \frac{12 \kappa_z (d_s+1) \log T}{\kappa_s \kappa_p} \cdot \frac{2^{ d_z ( M + 1 ) }}{2^{d_z } - 1} + 6 \cdot 2^{-M} \cdot T. 
    \end{align*} 
    % \end{equation*} 
    % As this inequality must hold for every $M^*$, by choosing 

    The above derivation is true for any choice of $M$. 
    Note that the right side of the inequality above reaches its minimum when $M = \frac{\log \left( \frac{\kappa_s \kappa_p T (2^{d_z} - 1)}{\kappa_z d_z (d_s+1) \cdot 2^{d_z + 1} \log T} \right) }{(d_z+1) \log 2} $.
    % , and the inequality must hold for any $M^*$, 
    We thus obtain
    \begin{equation*} 
        R(T) \leq \frac{12 (d_z + 1)}{d_z} \cdot \left( \frac{ \kappa_z d_z (d_s+1)}{\kappa_s \kappa_p (2^{d_z} - 1)} \right)^{\frac{1}{d_z + 1}} \cdot T^{\frac{d_z }{d_z + 1}} \cdot \left( \log T \right)^{\frac{1}{d_z +1}}. 
    \end{equation*} 
\end{proof}

\subsection{BLiN-MOS with Improved Communication Complexity} 

Recently, bandit problems with batched feedback has attracted the attention of many researchers (e.g., \cite{perchet2016batched,gao2019batched,han2020sequential,ruan2021linear,li2022gaussian,agarwal2022batched,fenglipschitz}). In such settings, the reward samples are not communicated to the player after each arm pull. Instead, the reward samples are collected in batches. In such settings, we not only want to minimize regret, but also want to minimize rounds of communications. In this section, we show that BLiN-MOS can achieve regret rate of order $ \mathcal{O} \( T^{\frac{d_z}{d_z+1}} \( \log (T) \)^{\frac{1}{d_z + 1}} \) $ with $\mathcal{O} \( \log \log T \)$ rounds of communications. 
% To reach this communication complexity, we use the idea of ACE sequence introduced by \cite{fenglipschitz}. 
A formal statement of this result is below in Theorem \ref{thm:ace-noise-total}. 
% The proof of Theorem \ref{thm:ace-noise-total} can be found in the Appendix. 

\begin{theorem}
    \label{thm:ace-noise-total}
    Let $T$ be the total time horizon and let $\epsilon = \frac{1}{T^2}$. Apply the sequence $\{r_m\}$ in Definition \ref{def:ace-noise} to Algorithm \ref{alg:large-d}. 
    % There exists an edge-length sequence $\{ r_m \}_{m=1}^M$ such that the followings are simultaneously satisfied: 
    Then it holds that 
    1. 
    with probability exceeding $1 - \frac{2}{T}$, the total regret of BLiN-MOS algorithm satisfies 
    \begin{align*}
        R(T) 
        \leq& \;  
        \Bigg[ c \cdot \left( \frac{\log \log \frac{T}{\log T} - \log(d_z +1) - \log \widetilde{C}}{\log \frac{d+1}{d+1-d_z}} +1 \right) 
        + 6 \cdot 2^{\widetilde{C}}  \Bigg] \cdot T^{\frac{d_z}{d_z+1} } \cdot \left( \log T\right)^{\frac{1}{d_z + 1}} 
        % \left[ c \frac{\log \log \frac{T}{\log T} - \log(d_z +1) - \log \widetilde{C}}{\log \frac{d+1}{d+1-d_z}}  + 6 \cdot 2^{\widetilde{C}} \right] \\
        % &\cdot T^{\frac{d_z}{d_z+1} } \cdot \left( \log T\right)^{\frac{1}{d_z + 1}}, 
    \end{align*}
    where $c$ is a constant independent of $T$, $\widetilde{C} $ is a constant satisfying $\widetilde{C} \geq \frac{d+1-d_z}{d_z}$, $d_z$ is the zooming dimension; and 2. BLiN-MOS only needs $ \mathcal{O}(\log\log T)$ rounds of communications to achieve this regret rate. 
\end{theorem}

% To prove Theorem \ref{thm:ace-noise-total}, we first consider the following edge-length sequence. 

\begin{definition} 
    \label{def:ace-noise} 
    Denote $c_1 = \frac{d_z  }{ (d+1)  (d_z+1)  }\log \frac{T}{\log T}$, $ c_{i+1} = \eta c_i,i\geq 1$, where $\eta = \frac{d+1-d_z}{d + 1 }$. Let $a_n = \lfloor \sum_{i=1}^n c_i\rfloor$, $b_n = \lceil \sum_{i=1}^n c_i\rceil$. Define sequence $\{r_m\}_m $ as $r_m = \min \{ r_{m-1}, 2^{-a_n} \}$ for $m = 2n-1$ and $r_m = 2^{-b_n}$ for $m = 2n$. 
    % (Here we suppose $\{r_m\}$ is a strictly decreasing sequence, see Remark \ref{rm:batch} for details). 
    % Define sequence $\{\alpha_m\}_m$ as $\alpha_m = r_m^k,k>0$. 
\end{definition} 

In words, Theorem \ref{thm:ace-noise-total} states that only $ \mathcal{O}(\log\log T) $ rounds of communications are needed for BLiN-MOS to achieve a regret of order $ \wt{\mathcal{O}} \( T^{ \frac{d_z}{d_z + 1} } \) $. Next we present a proof of this Theorem. 
% Further details can be found in the Appendix. 
% Some preparation lemmas for this the

\begin{proof}[Proof of Theorem \ref{thm:ace-noise-total}]
    Denote $\tilde{r}_n = 2^{-\sum_{i=1}^n c_i}$. Let $\widehat{M} = \frac{\log \log \frac{T}{\log T} - \log(d_z + 1) - \log \widetilde{C}}{\log \frac{d+1}{d+1-d_z}}$, where $\widetilde{C}$ is a constant satisfying $\widetilde{C} \geq \frac{d+1-d_z}{d_z}$. Then $c_{\widehat{M}} = \eta^{\widehat{M} - 1}c_1 \geq 1$. From the fact that $\{c_i\}$ is decreasing, we know $c_i \geq 1, i = 1,\cdots, \widehat{M}$ and thus $b_1 < b_2 < \cdots < b_{\widehat{M}}$. 
    %Below we prove that a total of $2 \widehat{M}$ batches are needed to achieve the total regret stated in Theorem \ref{thm:ace-noise-total}.
    
    We divide the total regret into two parts, the total regret of the first $2\widehat{M}$ batches and the total regret after $2\widehat{M}$ batches. We first consider the first $2\widehat{M}$ batches.
    
    \textbf{Case I}: $m = 2n-1$. In this case, we have $r_m \geq \tilde{r}_n, r_{m-1} \leq \tilde{r}_{n-1}$. 
    Since $ \Delta_x \leq 3 r_{m-1}$ for any $x \in \cup_{q_m\in \mathcal{A}_m}q_m$, any cube in $\mathcal{A}_{m-1}'$ is a subset of $S(3r_{m-1})$. Therefore, 
    \begin{equation*}
        |\mathcal{A}_{m-1}'| \leq N_{r_{m-1}} \leq \kappa_z r_{m-1}^{-d_z}. 
    \end{equation*}
    After partitioning $\mathcal{A}_{m-1}'$ into $\mathcal{A}_m$, we get
    \begin{equation*}
        |\mathcal{A}_m| = \left(\frac{r_{m-1}}{r_m}\right)^d |\mathcal{A}_{m-1}'| \leq \left(\frac{r_{m-1}}{r_m}\right)^d \cdot \kappa_z r_{m-1}^{-d_z}.
    \end{equation*}
    Denote the total regret of the $m$-th batch as $R_m$, then $R_m$ can be upper bounded as follows,
    % \begin{equation*}
    \begin{align}
        R_m =&\;
        \sum_{q_m\in \mathcal{A}_m} \sum_{i=1}^{n_m}\Delta_{x_{q_m,i}} \nonumber  \\ 
        \leq& \; 
        3 |\mathcal{A}_m| \cdot n_m \cdot r_{m-1} \nonumber \\
        \leq& \; 
        3 \kappa_z\left(\frac{r_{m-1}}{r_m}\right)^d \cdot r_{m-1}^{-d_z+1} \cdot \frac{\log (1/\epsilon) }{ \log \(1 - \frac{\kappa_s \kappa_p }{d_s + 1} r_m \beta^{d_s + 1} \) } , \label{eq:break}
    \end{align} 
    where the last inequality uses $ | \mathcal{A}_m | \le \kappa_z 2^{m d_z}$ and $n_m \le \frac{(d_s+1) \log (1/\epsilon)}{\kappa_s \kappa_p } 2^{ m} $ (See Proposition \ref{prop:basic} in the Appendix). We continue the above calculation, and obtain 
    \begin{align*}
        % \frac{(d_s+1) \log (1/\epsilon)}{\kappa_s \kappa_p \alpha_m^{ d_s+1 }} \\
        (\ref{eq:break})\leq& \; 
        \frac{3 \kappa_z (d_s+1) \log (1/\epsilon)}{\kappa_s \kappa_p} \cdot  r_m^{-d-1} \cdot r_{m-1}^{d-d_z + 1} \\ 
        % \tag{by Proposition \ref{prop:basic}} \\
        \leq& \; 
        \frac{3 \kappa_z (d_s+1) \log (1/\epsilon)}{\kappa_s \kappa_p} \cdot \widetilde{r}_n^{-d-1} \cdot \widetilde{r}_{n-1}^{d-d_z+1}\\
        =& \; 
        \frac{3 \kappa_z (d_s+1) \log (1/\epsilon)}{\kappa_s \kappa_p} \cdot 2^{d_z \sum_{i=1}^{n-1}c_i + (d+1)c_n}.
    \end{align*}
    % \end{equation*}
    Define $C_n = d_z \sum_{i=1}^{n-1}c_i + (d+1)c_n$, then $C_n - C_{n-1} = (d_z-d-1)c_{n-1} + (d+1)c_n = 0$. So for any $n > 1$, we have $C_n = C_1 = (d+1)c_1$. Consequently, we obtain the upper bound of $R_m$ for $m = 2n-1$,
    \begin{equation*}
        R_m \leq \frac{3 \kappa_z (d_s+1) \log (1/\epsilon)}{\kappa_s \kappa_p} \cdot \left(\frac{T}{\log T}\right)^{\frac{d_z}{d_z+1}}.
    \end{equation*}
    Summing over odd $m$, we have
    \begin{align}
        & \; \sum_{n = 1}^{\widehat{M}} R_{2n-1} 
        \leq 
        \frac{3 \kappa_z (d_s+1) \log (1/\epsilon)}{\kappa_s \kappa_p} \cdot \left(\frac{T}{\log T}\right)^{\frac{d_z}{d_z+1}} \cdot \widehat{M} \nonumber \\
        &=
        \frac{3 \kappa_z (d_s+1) \log (1/\epsilon)}{\kappa_s \kappa_p} \cdot \left(\frac{T}{\log T}\right)^{\frac{d_z}{d_z+1}} \nonumber \\
        &\;\;\;\; \cdot \frac{\log \log \frac{T}{\log T} - \log(d_z +1) - \log \widetilde{C}}{\log \frac{d+1}{d+1-d_z}} \label{ineq:case1} 
    \end{align}

    \textbf{Case II}: $m = 2n$. Note that $r_{m-1} = \min\{r_{m-2},2^{-a_n} \} = r_{m-2}$ happens only when $\lceil \sum_{i=1}^{n-1}c_i \rceil > \lfloor \sum_{i=1}^{n}c_i \rfloor$. If this strict inequality holds, then $r_m = 2^{-\lceil \sum_{i=1}^{n}c_i \rceil } = 2^{-\lfloor \sum_{i=1}^{n}c_i \rfloor - 1} \geq 2^{-\lceil \sum_{i=1}^{n-1}c_i \rceil } = r_{m-2}$, which contradicts the fact that $\{b_n\}$ is strictly increasing for $n = 1,\cdots,\widehat{M}$. Therefore, it must hold that $r_{m-1} = 2^{-a_n} = 2^{-b_n + 1} = 2r_m$. We thus conclude from Lemma \ref{lem:opt-gap} that any cube in $\mathcal{A}_m$ is a subset of $S(6 r_m)$.  Therefore,
     % \begin{equation*}
    \begin{align*}
        R_m 
        =&\; 
        \sum_{q_m\in \mathcal{A}_m} \sum_{i=1}^{n_m}\Delta_{x_{q_m,i}} 
        \leq 
        6 |\mathcal{A}_m| \cdot n_m \cdot r_m \\
        \leq&\; 6 \kappa_z \cdot r_m^{-d_z+1} \cdot \frac{(d_s+1) \log(1/\epsilon)}{\kappa_s \kappa_p r_m} \\
        =&\; \frac{6 \kappa_z (d_s+1) \log (1/\epsilon) }{\kappa_s \kappa_p} \cdot r_m^{-d_z}.
    \end{align*}
    % \end{equation*}

    Since $r_m \geq 2r_{m+2}$ when $m$ is even, we obtain
    \begin{equation*}
    \begin{aligned}
        && \sum_{n = 1}^{\widehat{M}} R_{2n} & \leq \sum_{n = 1}^{\widehat{M}} \frac{6 \kappa_z (d_s+1) \log (1/\epsilon) }{\kappa_s \kappa_p} \cdot r_{2n}^{-d_z} \\
        && & \leq \frac{6 \kappa_z (d_s+1) \log (1/\epsilon) }{\kappa_s \kappa_p} \cdot r_{2\widehat{M}}^{-d_z} \cdot \sum_{n=0}^{\widehat{M}-1} \left( \frac{1}{2^n} \right)^{d_z} \\
        && & \leq \frac{6 \kappa_z (d_s+1) \log (1/\epsilon) }{\kappa_s \kappa_p} \cdot \frac{1}{1 - \left(\frac{1}{2}\right)^{d_z}} \cdot r_{2\widehat{M}}^{-d_z}
    \end{aligned}
    \end{equation*}
    Note that $\widehat{M} = \frac{\log \log \frac{T}{\log T} - \log(d_z +1) - \log \widetilde{C}}{\log \frac{d+1}{d+1-d_z}}$, we have $r_{2\widehat{M}} = 2^{-\lceil \sum_{i=1}^{\widehat{M}} c_i \rceil}  = 2^{- \lceil c_1 \cdot \frac{1-\eta^{\widehat{M}}}{1-\eta} \rceil } = 2^{-\lceil \frac{\log \frac{T}{\log T}}{d_z + 1} -\widetilde{C} \rceil} \geq \left(\frac{T}{\log T} \right)^{-\frac{1}{d_z + 1}}$. Therefore, it holds that 
    \begin{equation}
    \label{ineq:case2}
        \sum_{n = 1}^{\widehat{M}} R_{2n} \leq \frac{6 \kappa_z (d_s+1) \log (1/\epsilon) }{\kappa_s \kappa_p \( 1 - \left(\frac{1}{2}\right)^{d_z} \) } \cdot \left(\frac{T}{\log T} \right)^{\frac{d_z}{d_z + 1}}.
    \end{equation}

    Additionally, for the total regret after $2\widehat{M}$ batches, $r_{2\widehat{M}} = 2^{-\lceil \frac{\log \frac{T}{\log T}}{d_z + 1} -\widetilde{C} \rceil} \leq 2^{- \frac{\log \frac{T}{\log T}}{d_z + 1} + \widetilde{C} }$ gives
    \begin{align}
    \label{ineq:case3}
        \sum_{m > 2\widehat{M}} R_m & \leq 6 r_{2\widehat{M}} \cdot T 
        % \nonumber 
        \leq 6 \cdot 2^{\widetilde{C}}  \left(\frac{T}{\log T} \right)^{-\frac{1}{d_z + 1}}  T.
    \end{align}

    By replacing $\epsilon$ with $\frac{1}{T^2}$ and combining inequalities (\ref{ineq:case1}), (\ref{ineq:case2}), and (\ref{ineq:case3}), we obtain
    % \begin{equation*}
    \begin{align*}
        R(T) 
        \leq& \;  
        \sum_{n = 1}^{\widehat{M}} R_{2n-1} + \sum_{n = 1}^{\widehat{M}} R_{2n} + \sum_{m > 2\widehat{M}} R_m\\
        % && & \leq \left[ c \cdot \frac{ \kappa_z (d_s+1)}{\kappa_s \kappa_p} \cdot \left( \frac{\log \log \frac{T}{\log T} - \log(d_z +1) - \log \widetilde{C}}{\log \frac{d+1}{d+1-d_z}} +1 \right) + 6 \cdot 2^{\widetilde{C}}  \right] \cdot T^{\frac{d_z}{d_z+1} } \cdot \left( \log T\right)^{\frac{1}{d_z + 1}},
        \leq&\; 
        \Bigg[ c \cdot \left( \frac{\log \log \frac{T}{\log T} - \log(d_z +1) - \log \widetilde{C}}{\log \frac{d+1}{d+1-d_z}} +1 \right) 
        + 6 \cdot 2^{\widetilde{C}}  \Bigg] \cdot T^{\frac{d_z}{d_z+1} } \cdot \left( \log T\right)^{\frac{1}{d_z + 1}},
    \end{align*}
    % \end{equation*}
    where $c$ is a constant independent of $T$. 
    Finally, we conclude that the BLiN-MOS uses $\mathcal{O} \( \log \log T\)$ rounds of communications by noticing that $ \wh{M} = \mathcal{O} \( \log \log T\) $. 

    % Finally, the analysis above shows that a total regret of order $\mathcal{O} \left( T^{\frac{d_z}{d_z+1} } \cdot \left( \log T\right)^{\frac{1}{d_z + 1}} \right) $ can be achieved if the for-loop in Algorithm \ref{alg:large-d} runs $2\widehat{M}$ times and the remaining budget is used for arbitrary pulling. Therefore, $2\widehat{M} + 1$ rounds of communications can guarantee such total regret rate.
\end{proof}

%% file: tex/conslusion.tex
\section{Conclusion}

In this paper, we introduce the concept of scattering dimension as a tool for quantifying the performance of the random search algorithm \cite{bergstra2011algorithms,bergstra2012random}. Our results provide the first theoretical grounding for the random search algorithm. We study the properties of scattering dimension and discuss its connection to the classic concept of zooming dimension. Following this, we investigate the problem of continuum-armed spaces with bounded and positively supported noise. 
For such problems, we introduce a new bandit algorithm, BLiN-MOS, that achieves regret rate of order $\wt{\O } \( T^{\frac{d_z}{d_z+1}} \)$ in compact doubling metric spaces with a probability measure, where $d_z$ is the zooming dimension.

% In this paper, we introduce the concept of scattering dimension as a tool of quantifying the performance of the random search algorithm. Also, we introduce a new Lipschitz bandit algorithm, BLiN-MOS, that achieves regret rate of order $\wt{\O } \( T^{\frac{d_z}{d_z+1}} \)$ in compact doubling metric spaces with a probability measure. 
% This result breaks the lower bound in (compact) metric spaces and suggests a fundamental difference between metric spaces and metric measure spaces. 

%% file: tex/appendix.tex
% \section{Proof of Theorem \ref{thm:ace}} 

% \input{tex/blie-theory} 

% \section{Additional Lemmas and Their Proofs} 

\clearpage

\section{Omitted Proofs}

\begin{proposition}
    \label{prop:basic}
    Let $\kappa \in (0,1]$ and $\gamma > 1$. 
    For any $ x \in [1,\infty) $, it holds that 
    \begin{align*}
        - \frac{\gamma^{x}}{\gamma \log (1 - \frac{\kappa}{\gamma}) } \le - \frac{1}{\log(1 - \kappa \gamma^{-x})} \le \frac{\gamma^{x}}{\kappa}. 
    \end{align*} 

\end{proposition}

\begin{proof}[Proof of Proposition \ref{prop:basic}] 
    Consider function $h: [1, \infty ) \to \R$ defined as 
    \begin{align*} 
        h (x) = - \frac{\kappa }{ \gamma^x \log(1 - \kappa \gamma^{-x}) } . 
    \end{align*} 
    Clearly, $h(1) = - \frac{\kappa }{ \gamma \log (1 - \frac{\kappa}{\gamma}) } $. By noting $ \log(1 - \kappa \gamma^{-x}) = - \kappa \gamma^{-x} + o (\gamma^{-x})$ when $x$ is large, we know that 
    \begin{align*}
        \lim_{x \to \infty} h (x) = 1. 
    \end{align*} 
    Also, we know that 
    \begin{align*}
        h' (x) = \frac{\kappa \gamma^x \log \gamma}{ (\gamma^x \log (1 - \kappa \gamma^{-x}) )^2 } \( \log \( 1 - \kappa \gamma^{-x} \) + \frac{\kappa \gamma^{-x}}{1 - \kappa \gamma^{-x}} \). 
    \end{align*}
    Since $ \log (1-z) \ge \frac{-z}{\sqrt{1-z}} $ for all $z \in (0,1)$, we know $ h' (x) > 0 $ for all $x \in [1,\infty)$. 
    Thus $ h $ is increasing over $ [1,\infty) $. Therefore $ - \frac{\gamma^{x}}{\gamma \log (1 - \frac{\kappa}{\gamma}) } \le - \frac{1}{\log(1 - \kappa \gamma^{-x})} \le \frac{\gamma^{x}}{\kappa} $ for all $x \in [1,\infty)$. 
\end{proof}

\begin{proof}[Proof of Proposition \ref{prop:gp}] 
    Clearly, the maximum of $ g_p $ is attained at $0$. Thus it suffices to consider cubes $q$ such that $q = [0,r]^d$. 
    For $q = [0,r]^d$, we have $f^{\max}_q = 1 ,f^{\min}_q = 1 - \frac{1}{p} r^p$, $(1-\alpha)f^{\max}_q +\alpha f^{\min}_q = 1 - \frac{1}{p} r^p \alpha$. Let $X_q$ denote the uniform random variable over $q$, and let $X_{q,j}$ ($1 \le j \le d$) be the $j$-th component of $X_q$. Then it holds that 
    \begin{equation*} 
        \begin{aligned} 
        && \mathbb{P} \left(f(X_q) > (1-\alpha)f^{\max}_q +\alpha f^{\min}_q \right) & = \mathbb{P} \left( \| X_q \|_{\infty}^p < r^p \alpha \right) \\
        && & = \prod_{j=1}^d \mathbb{P} \left( | X_{q,j} | < r \alpha^{\frac{1}{p}}  \right) \\
        && & =\alpha^{\frac{d}{p}} . 
        \end{aligned} 
    \end{equation*} 
    Thus $ \mathbb{P} \left(f(X_q) \le (1-\alpha)f^{\max}_q +\alpha f^{\min}_q \right) = 1 - \mathbb{P} \left(f(X_q) > (1-\alpha)f^{\max}_q +\alpha f^{\min}_q \right) = 1 - \alpha^{\frac{d}{p}} $. 
    % which conclude
\end{proof} 

\begin{proof}[Proof of Theorem \ref{thm:rs}] 
    % Recall that $ f_S^{\max} $ denote the supremum 
    It holds that 
    \begin{align} 
        \Pr \( f^* - Y_{T}^{\max} > \alpha L \Theta  \) 
        \le& \;  
        \Pr \( f^{*} - Y_{T}^{\max} > \alpha \( f_{ \X }^{\max} - f_{ \X }^{\min} \) \) \tag{by $L$-Lipschitzness} \nonumber \\ 
        =& \; 
        \Pr \( Y_{T}^{\max} < f_{ \X }^{\max} - \alpha \( f_{ \X }^{\max} - f_{ \X }^{\min} \) \) \nonumber \\ 
        =& \; 
        \prod_{i=1}^T \Pr \( f ( x_{i} ) < f_{ \X }^{\max} - \alpha \( f_{ \X }^{\max} - f_{ \X }^{\min} \) \) \nonumber \\ 
        \le& \;  
        \( 1 - \kappa_s \alpha^{d_s} \)^T , \label{eq:mos}
        % \tag{by Assumption \ref{assumption:ord} and that $\{ X_{q^*}^i\}_{i=1}^n $ are $iid$} 
    \end{align} 
    where the last inequality follows from the definition of scattering dimension and that $\{ x_i\}_{i=1}^n $ are $iid$. 

    By letting $\epsilon = \( 1 - \kappa_s \alpha^{d_s} \)^T  $, we know that $ \Pr \( f^* - Y_{T}^{\max} > L \Theta \( \frac{1 - e^{ \frac{\log \epsilon }{T} }}{ \kappa_s } \)^{ \frac{1}{d_s} } \) \le \epsilon $. 
\end{proof}

\begin{proof}[Proof of Theorem \ref{thm:rs-asymp}]
    Fix an arbitrary sequence $\{ \omega_T \}_T$ with $\lim\limits_{T \to \infty } \omega_T = \infty$ and $\omega_T = o (T)$. 
    For any $T \in \mathbb{N}$, we let $ \alpha_T = \frac{ \( \frac{\omega_T}{T} \)^{\frac{1}{d_s}}  }{L \Theta } $. Then by the same argument leading to (\ref{eq:mos}), we have, for any $\beta > 0$, 
    \begin{align*} 
        \Pr \( \frac{ \( \frac{T}{\omega_T} \)^{\frac{1}{d_s}}  }{L \Theta } \( f^* - Y_{T}^{\max} \) > \beta \) 
        = 
        \Pr \( f^* - Y_{T}^{\max} > \beta L \Theta \( \frac{\omega_T}{T} \)^{\frac{1}{d_s}} \) 
        \le 
        \( 1 - \kappa_s \beta^{d_s} \frac{\omega_T}{T} \)^T . 
    \end{align*} 
    Since 
    \begin{align*}
        \lim_{T \to \infty} \( 1 - \kappa_s \beta^{d_s} \frac{\omega_T}{T} \)^T 
        = 
        \lim_{T \to \infty} e^{ - \kappa_s \beta^{d_s} \omega_T } = 0,  
    \end{align*}
    we have $ \lim\limits_{T \to \infty} \Pr \( \frac{ \( \frac{T}{\omega_T} \)^{\frac{1}{d_s}}  }{L \Theta } \( f^* - Y_{T}^{\max} \) > \beta \) = 0$ for all $\beta > 0$. Thus $ \( \frac{ T }{\omega_T} \)^{\frac{1}{d_s}} \( f^* - Y_{T}^{\max} \) \overset{P}{\to} 0 $. 
    
\end{proof}

\begin{proof}[Proof of Theorem \ref{thm:rs-noisy}]

Let $X$ denote the Borel random variable defined in $ \X $ so that the law of $X$ is $\mu$. Let $Z = b - W$. 
Then we have, for any $\beta > 0$, 
% -----------------------------------------------------------------
\begin{align*} 
    & \; \Pr \( \frac{ f^* - f ( X ) }{ L \Theta } + Z > \beta \) \\
    =& \;  
    \int_{-\infty}^{\infty} f_Z ( z ) \Pr \( \frac{ f^* - f ( X ) }{L \Theta } > \beta - z \) \,  dz \\ 
    =& \;  
    \int_{-\infty}^{\beta - 1 } f_Z ( z ) \Pr \( \frac{ f^* - f ( X ) }{L \Theta } > \beta - z \) \,  dz + \int_{\beta - 1 }^{\beta} f_Z ( z ) \Pr \( \frac{ f^* - f ( X ) }{L \Theta } > \beta - z \) \,  dz \\
    &+ \int_{\beta }^{\infty} f_Z ( z ) \Pr \( \frac{ f^* - f ( X ) }{L \Theta } > \beta - z \) \,  dz \\ 
    \le& \; 
    % =& \; 
    % \int_{-\infty}^{\beta - 1 } f_Z ( z ) \,  dz + 
    \int_{\beta - 1}^{\beta } f_Z ( z ) \( 1 - \kappa_s \( \beta - z \)^{d_s} \) \, dz +  \int_{\beta }^{ \infty } f_Z ( z ) \, dz , 
    % \le& \; 
    % 1 - \kappa_s \int_{a}^b f_Z (z) \( \beta - z  \)^{d_s} \, dz 
\end{align*} 
where the last line uses (1) $ \Pr \( \frac{ f^* - f ( X ) }{L \Theta } > \beta - z \) = 0 $ if $ z \le \beta - 1 $, (2) $ \Pr \( \frac{ f^* - f ( X ) }{L \Theta } > \beta - z \) = 1 $ if $ z \ge \beta $. 

% Let $X_1, X_2, \cdots, X_T$ be $iid$ copies of $X$.  

Since $W$ is positively supported on $[a,b]$, $Z$ is positively supported on $[0,b-a]$ and there exists $ \kappa_p >0 $ such that $ f_Z (z) \ge \kappa_p $ for all $z \in [0,b-a]$. We continue the above calculation and get, for $\beta \in (0,b-a] $, 
\begin{align} 
    \Pr \( \frac{ f^* - f ( X ) }{L \Theta } + Z > \beta \) 
    \le& \;  
    \Pr \( Z \ge \beta - 1 \) - \kappa_s \int_{ \beta - 1 }^{\beta } f_Z ( z ) \( {\beta - z} \)^{d_s} \, dz \nonumber \\ 
    \le& \; 
    1 - \kappa_s \kappa_p \int_{ [\beta - 1, \beta] \cap [0,b-a] } \( {\beta - z} \)^{d_s} \, dz \nonumber \\ 
    =& \; 
    1 - \kappa_s \kappa_p \int_{ 0 }^{\beta} \( {\beta - z} \)^{d_s} \, dz \nonumber \\ 
    =& \; 
    1 - \frac{ \kappa_s \kappa_p }{ d_s + 1 } \beta^{d_s + 1} . \label{eq:noisy-mos-bound} 
    % 1 - \frac{ \kappa_s \kappa_p }{ d_s + 1 } 
    % \le 
    % \Pr \( Z \ge \beta \) . 
    % \le& \;  
    % 1 - \kappa_s \int_{ \beta - 1 }^{\beta } f_Z ( z ) \( {\beta - z} \)^{d_s} \, dz . 
\end{align}

Now consider $ X_1, X_2, \cdots, X_T $ that are $iid$ copies of $X$, and $ Z_1, Z_2, \cdots Z_T $ that are $iid$ copies of $Z$. 
From the above derivation, we have, for sufficiently small $\beta$, 
\begin{align*}
    \Pr \( \( \frac{ T }{\omega_T} \)^{ \frac{1}{ d_s + 1} } \min_{ i: 1\le i \le T } \left\{ \frac{ f^* - f ( X_i ) }{L \Theta } + Z_i \right\} > \beta  \) 
    \le   
    \( 1 - \frac{\kappa_s \kappa_p }{d_s + 1} \beta^{d_s+1} \frac{\omega_T}{T} \)^T . 
\end{align*} 

Taking limits on both sides of the above inequality gives 
\begin{align*}
    \lim_{T \to \infty} \Pr \( \( \frac{ T }{\omega_T} \)^{ \frac{1}{ d_s + 1} }  \min_{ i: 1\le i \le T } \left\{ \frac{ f^* - f ( X_i ) }{L \Theta } + Z_i \right\} > \beta  \) = 0, 
\end{align*}
for any positive $ \beta $ in a small neighborhood of zero. This concludes the proof.

\end{proof}

\begin{proof}[Proof of Lemma \ref{lem:not-eli}]
    Let $q_m^* $ be the cube in $\mathcal{A}_m$ that contains $ x^* $. For $m = 1,2,\cdots$, define the following events: 
    \begin{align*}
        \mathcal{E}_m := \{ \text{There exist a cube $q_m^* \in \mathcal{A}_m$ such that $q_m^* \ni x^*$ } \}.  
    \end{align*} 
    % If $\mathcal{E}_s$ is true, we have 
    % \begin{align*} 
    %     f^* - \wh{\mu}_m^{\max} 
    %     \le 
    %     f^* - \wh{\mu}_m ( q_m^* ) 
    %     \le 
    %     f^* - \mu (x^*) + r_m + \sqrt{ \frac{2 \log T }{n_m^z} } 
    %     = 2 r_m , 
    % \end{align*} 
    % which means $x^*$ is not eliminated by the zooming elimination rule. 

    We will show that $ \Pr \( \mathcal{E}_{m+1} | \mathcal{E}_{m} \) \ge 1 - \epsilon $. Before proceeding that, we let $ \mathrm{Diam} (q^*) $ denote the diameter (edge-length) of cube $q^*$, and note that 
    % 
    % by the scattering elimination rule. 
    % Let $ \mu_q := \E_{x\sim \mathrm{Unif} (q)} \[ f (q) + W \]  $. 
    setting $Z = \frac{ 1 - W }{ \mathrm{Diam} (q^*) } $ in Eq. (\ref{eq:noisy-mos-bound}) (in the Appendix) gives that for the $q^*$ that contains $x^*$: 
    \begin{align} 
        \Pr \( \frac{ f^* + 1 - y_{q^*,i} }{ \mathrm{Diam} (q^*) } \ge \beta \) 
        \le& \; 
        1 - \frac{ \kappa_s \kappa_p }{ d_s + 1 } \mathrm{Diam} (q^*) \beta^{d_s + 1} , \label{eq:scattering-ind} 
    \end{align} 
    since $ \frac{\kappa_p}{\mathrm{Diam} (q^*)} $ lower bounds the noise density of $ \frac{ 1 - W }{ \mathrm{Diam} (q^*) } $.   
    % \textcolor{red}{(maybe $L$ can be omitted in this part)}

    % Now we start the proof for . 
    % Note that $\mathcal{E}_1$ is surely true. 
    Now suppose that $ \mathcal{E}_m $ is true. Then letting $q^* = q_m^*$ in (\ref{eq:scattering-ind}) gives $ \Pr \(  f^* + 1 - y_{q^*,i}  \ge \beta r_m \) \le 1 - \frac{ \kappa_s \kappa_p }{ d_s + 1 } r_m \beta^{d_s + 1} $. By $iid$-ness of $y_{q_m^*, 1}, \cdots, y_{q_m^*, n_m } $ and noticing that $n_m = \frac{\log (1/\epsilon)}{\log \left(1 - \frac{\kappa_s \kappa_p}{d_s+1}r_m \beta^{d_s+1} \right)}$, we know that, with probability exceeding $1 - \epsilon$, for each $m$ 
    \begin{align*} 
        f^* + 1 - Y_{q_m^*, n_m} 
        \le 
        \beta r_m . 
        % \frac{ \log (1/\epsilon) }{ \log \( 1 - \frac{\kappa_s \kappa_p }{ d_s + 1 } r_m \beta^{d_s + 1} \) } . 
        % \le 
        % \frac{ \log (1/\epsilon) }{   } 
    \end{align*} 
    Setting $\beta = 1$ in the above equation gives that, for each $m$,  
    \begin{align*} 
        f^* + 1 - Y_{q_m^*, n_m}
        \le r_m 
    \end{align*} 
    happens with probability exceeding $1 - \epsilon$. 

    Then from the elimination rule, we know that, with probability exceeding $1 - \epsilon$, 
    \begin{align*} 
        Y_m^{\max} - Y_{q_m^*, n_m} 
        \le 
        Y_{m}^{\max} - (f^* + 1 ) + f^* + 1 - Y_{q_m^*, n_m} 
        \le 
        r_m . 
    \end{align*} 
    This means $ \mathcal{A}_{m}' $ contains the cube $q_m^*$. Thus under the event $ \mathcal{E}_m $, $ \mathcal{E}_{m+1} $ holds true with probability no smaller than $ 1 - \epsilon $. 

    If $x^*$ is contained in some cube in $\mathcal{A}_m$, then it must survive all previous eliminations. Writing this down gives $ \mathcal{E}_{m} \cap \mathcal{E}_{m'} = \mathcal{E}_{m} $ whenever $ m \ge m' $. Therefore, 
    % Next, note that $ \mathcal{E}_{m} \cup \mathcal{E}_{m'} = \mathcal{E}_{m} $ if $ m \ge m' $ ()
    \begin{align*}
        \Pr \( \cap_{m=1}^T \mathcal{E}_m \)
        =& \;  
        \Pr \( \mathcal{E}_T | \cap_{m=1}^{T-1} \mathcal{E}_m \) \Pr \( \mathcal{E}_{T-1} | \cap_{m=1}^{T-2} \mathcal{E}_m \) \cdots \Pr \( \mathcal{E}_2 | \mathcal{E}_1 \) \Pr \( \mathcal{E}_1 \) \\ 
        =& \; 
        \Pr \( \mathcal{E}_T | \mathcal{E}_{T-1} \) \Pr \( \mathcal{E}_{T-1} | \mathcal{E}_{T-2} \) \cdots \Pr \( \mathcal{E}_2 | \mathcal{E}_1 \) \Pr \( \mathcal{E}_1 \)
        \ge 
        (1 - \epsilon )^T . 
    \end{align*} 

    Since the number of batches is clearly bounded by $T$, letting $\epsilon = \frac{1}{T^2}$ finishes the proof. 
    % $\epsilon = \frac{1}{T^2}$, we know that $ \Pr \(  \) $
    % Now the induction continues and we 
    % happens with probability . 
    
\end{proof}

\begin{proof}[Proof of Lemma \ref{lem:event}]
    Let 
    % \begin{align*}
    $
        \mathcal{E}' := \left\{ \text{$x^*$ is not eliminated during a $T$-step run } \right\} $. 
    By Lemma \ref{lem:not-eli}, we know that $ \Pr \( \mathcal{E}' \) \ge \( 1 - \frac{1}{T^2} \)^T $. Given $\mathcal{E}'$, an argument leading to (\ref{eq:noisy-mos-bound}) gives that 
    \begin{align*} 
        \Pr \( f^* + 1 - Y_{q_m^*, n_m} \le r_m \) \ge 1 - \epsilon . 
    \end{align*} 
    Then conditioning on $\mathcal{E}'$, for any $m$, with probability exceeding $1 - \epsilon$, we have 
    \begin{align*} 
        f^* + 1 - Y_{m}^{\max} 
        \le 
        f^* + 1 - Y_{q_m^*, n_m} + Y_{q_m^*, n_m} - Y_{m}^{\max} 
        \le 
        r_m . 
    \end{align*} 
    Since there are no more than $T$ batches, a union bound and $ \epsilon = \frac{1}{T^2} $ gives that $ \Pr \( \mathcal{E} | \mathcal{E}' \) \ge 1 - \frac{1}{T} $. We conclude the proof by noticing $ \Pr \( \mathcal{E} \) = \Pr \( \mathcal{E} \cap \mathcal{E} '  \) = \Pr \( \mathcal{E} | \mathcal{E}' \) \Pr \( \mathcal{E}' \) \ge \( 1 - \frac{1}{T^2} \)^T \( 1 - \frac{1}{T} \) \ge 1 - \frac{2}{T} $. 
\end{proof}

\begin{proof}[Proof of Lemma \ref{lem:opt-gap}]
    %\textcolor{red}{(check this)} 
    Fix any $x $ that is covered by cubes in $\mathcal{A}_m$. 
    Let $q_m$ be a cube in $\mathcal{A}_m$ such that $x \in q_m$. Let $f_{q_m}^{\max} := \sup_{x \in q_m} f (x)$. Let $q_m^*$ be the cube in $\mathcal{A}_m$ such that $x^* \in q_m^*$. 
    % By Lemma \ref{lem:event} and Lipschitzness of $f$, 
    Let $ W_i $ denote $iid$ copies of the noise random variable $W$. We have, for any $m$, 
    \iffalse
    \begin{equation*} 
        |f(x) - Y_{q_m,n_m} | 
        = 
        |f(x) - \max_{i:1\le i \le n_m} \( f(x_{q_m, i}) + W_i \) |
        \le 
        \le 
        | f (x) - f_{q_m}^{\max} | + |  f_{q_m}^{\max} + 1 - Y_{q_m,n_m} | 
        \le 2 r_m , 
        % \leq r_m + 1 \; \text{ and } \; f^* 
        % + 1 - Y_{m}^{\max} \le r_m . 
        % \leq |f(x) - f^{\max}_{q_m}| + |f^{\max}_{q_m} + 1 - Y_{q_m,n_m}| 
        % \leq r_m . 
    \end{equation*} 
    and by Lemma \ref{lem:event}, 
    \begin{align*} 
        f^* + 1 - Y_{m}^{\max} 
        \le 
        f^* + 1 - Y_{q_m^*, n_m} 
        \le 
        r_m . 
    \end{align*} 
    \fi
    \begin{align*} 
        |f(x) - Y_{q_m,n_m}| 
        =& \;  
        \left| f(x) - \max_{i:1\le i \le n_m} \( f(x_{q_m, i}) + W_i \) \right| \\ 
        \le& \;   
        \max_{i:1\le i \le n_m} | f(x) - f (x_{q_m,i}) | + \max_{i:1\le i \le n_m} | W_i | 
        \overset{(i)}{\leq} r_m + 1 , 
    \end{align*} 
    and by Lemma \ref{lem:event}, 
    \begin{align*} 
        f^* + 1 - Y_{m}^{\max} 
        \le 
        f^* + 1 - Y_{q_m^*, n_m} 
        \overset{(ii)}{\le} 
        r_m . 
    \end{align*} 

    Denote by $Par(q_m)$ the cube in $\mathcal{A}_{m-1}$ which contains $q_m$. Since $x \in q_m$, we know $x \in Par(q_m)$. Hence, it follows from the elimination rule that $Y_{m-1}^{\max} - Y_{Par(q_m),n_{m-1}} \leq r_{m-1}$. Therefore, by $(i)$ and $(ii)$, 
    % \begin{equation*} 
    % \begin{aligned} 
    % \begin{align*} 
    %     f^* - f(x) 
    %     \leq& \;  Y_{m-1}^{\max} + r_{m-1} - Y_{Par(q_m),n_{m-1}} + r_{m-1} 
    %     \leq 3 r_{m-1} .  
    %     % 3 r_{m-1} + r_{m-1}  \leq 4 r_{m-1} . 
    % \end{align*}
    \begin{align*}
    % \textcolor{red}{
        f^* - f(x) 
        \leq Y_{m-1}^{\max} + r_{m-1} - 1 - Y_{Par(q_m),n_{m-1}} + r_{m-1} + 1
        \leq 3 r_{m-1} .
    \end{align*}
        
    % \end{aligned} 
    % \end{equation*} 
    
\end{proof}

\section{Additional Note on the Relation between Scattering Dimension and Zooming Dimension}

Conforming to the zooming dimension via standard cubes (commonly known as dyadic cubes), we introduce the following version of scattering dimension via standard cubes. In the measure measure space $ (\mathcal{X}, \| \cdot \|_\infty, \Pr) $, $\mathcal{X} \subseteq \R^d$ is compact, and $\Pr$ is a probability measure over $\mathcal{X}$. Then the scattering dimension of a function $f$ defined over $\mathcal{X}$ is defined as  
\begin{align*} 
    d_s := \inf \{ \tilde{d} \ge 0 :&\; \exists \kappa \in (0,1], \text{ such that } \Pr \left( f ( X_B ) < f_B^{\max} -  \alpha ( f_B^{\max} - f_B^{\min} ) \right) 
    \le 
    1 - \kappa \alpha^{ \tilde{d} }, \\ 
    &\; \forall \text{ dyadic cube } B \subseteq \X \text{ with $x^* \in B$},\;  \forall \alpha \in (0, {1} ] \},
\end{align*} 
and use $B^*_h$ to denote the standard cube with edge length $h$ that contains $x^*$. Then we have the following proposition that relates $d_s$ to $d_z$. 
\begin{proposition}
    Assume that for any $h\leq1$, $B^*_h$ is a $r$-optimal region for some $r$. If the inequality in the definition of $d_z$ is tight, that is, the $r$-optimal region equals to the union of $cr^{-d_z}$ standard cubes with edge length $r$, then we have $d_z+d_s=d$.
\end{proposition}
\begin{proof}
    With our assumptions in place, we will calculate the scattering dimension from the zooming dimension. The standard cube $B^*_h$ has volume $m(B^*_h)=h^d$. Since it is an $r$-optimal region, we have
    \begin{align*}
        m(B^*_h)=ar^{-d_z}\cdot r^d=cr^{d-d_z}.
    \end{align*}
    Then we have 
    \begin{align*}
        h^d=m(B^*_h)=cr^{d-d_z},
    \end{align*}
    which yields that $r=c^{-\frac{1}{d-d_z}}\cdot h^\frac{d}{d-d_z}$. Therefore, we have $f_{B^*_h}^{\min}=1-c^{-\frac{1}{d-d_z}}\cdot h^\frac{d}{d-d_z}$. Substituting $f^{\max}_{B^*_h}=1$ gives that 
    \begin{align*}
        f_{B^*_h}^{\max} -  \alpha ( f_{B^*_h}^{\max} - f_{B^*_h}^{\min} )=1-\alpha c^{-\frac{1}{d-d_z}}\cdot h^\frac{d}{d-d_z},
    \end{align*}
    and
    \begin{align*}
        \Pr\left(f(X_{B^*_h})<f_{B^*_h}^{\max} -  \alpha ( f_{B^*_h}^{\max} - f_{B^*_h}^{\min} )\right)
        =1-\frac{m\left(\left\{x\in B^*_h:\;f(x)\geq1-\alpha c^{-\frac{1}{d-d_z}}\cdot h^\frac{d}{d-d_z}\right\}\right)}{h^d}
    \end{align*}
    Since $B^*_h$ equals to $S(c^{-\frac{1}{d-d_z}}\cdot h^\frac{d}{d-d_z})$, we have
    \begin{align*}
        \left\{x\in B^*_h:\;f(x)\geq1-\alpha c^{-\frac{1}{d-d_z}}\cdot h^\frac{d}{d-d_z}\right\}=S\left(\alpha c^{-\frac{1}{d-d_z}}\cdot h^\frac{d}{d-d_z}\right),
    \end{align*}
    and thus
    \begin{align*}
        m\left(\left\{x\in B^*_h:\;f(x)\geq1-\alpha c^{-\frac{1}{d-d_z}}\cdot h^\frac{d}{d-d_z}\right\}\right)=c\cdot\left(\alpha c^{-\frac{1}{d-d_z}}\cdot h^\frac{d}{d-d_z}\right)^{d-d_z}=\alpha^{d-d_z}\cdot h^d.
    \end{align*}
    Consequently, we have
    \begin{align*}
        \Pr\left(f(X_{B^*_h})<f_{B^*_h}^{\max} -  \alpha ( f_{B^*_h}^{\max} - f_{B^*_h}^{\min} )\right)=1-\alpha^{d-d_z}
    \end{align*}
    and $d_s=d-d_z$.
\end{proof}